\documentclass[10pt]{article} % For LaTeX2e
\usepackage[preprint]{tmlr}
\usepackage{hyperref}
\usepackage{url}

\usepackage[utf8]{inputenc}
\usepackage{graphicx}
\usepackage{amsmath}
\usepackage{amsthm}
\usepackage{booktabs}

\usepackage{amssymb,amsmath,amsthm, amsfonts}
\usepackage{thmtools, thm-restate}

\usepackage{tikz}

\usepackage{multirow}

\usepackage{bbm}
\usepackage{subfigure}

\usepackage{algorithm}
\usepackage{algorithmic}
\usepackage{MnSymbol}
\usepackage{xspace}

\newcommand{\formatcode}[1]{\textsl{#1}\xspace}
\newcommand{\algo}{\formatcode{ColME}}
\newcommand{\etaalgo}{\formatcode{$\eta$-ColME}}
\newcommand{\RR}{\formatcode{Round-Robin}}
\newcommand{\RRR}{\formatcode{Restricted-Round-Robin}}
\newcommand{\SRRR}{\formatcode{Soft-Restricted-Round-Robin}}
\newcommand{\ARRR}{\formatcode{Aggressive-Restricted-Round-Robin}}
\newcommand{\Choose}{\formatcode{choose\_agent}}
\newcommand{\Local}{\formatcode{Local}}
\newcommand{\Oracle}{\formatcode{Oracle}}
\newcommand{\rev}[1]{#1}

\newcommand{\nagent}{A}

\newcommand{\errorprob}{\delta}

\newcommand{\rv}[3]{#1_{#2}^{#3}}
\newcommand{\nobservation}[2]{n_{#1}^{#2}}

\newcommand{\oneagentindex}{a}
\newcommand{\onememoryindex}{l}

\newcommand{\CB}[2]{\beta_{{#1}}(#2)}
\newcommand{\CBinv}[2]{\beta^{-1}_{{#1}}(#2)}

\newcommand{\weighting}[2]{\alpha_{#1}^{#2}}

\newcommand{\distance}[3]{d_{#1}^{#2}({#3})}

\newcommand{\stoppingtime}[1]{\tau_{#1}}
\newcommand{\switchingtime}[1]{\zeta_{#1}}

\newcommand{\indicator}[1]{\mathbbm{1}_{\{#1\}}}
\newcommand{\class}[1]{\mathcal{C}_{#1}}
\newcommand{\classt}[2]{\mathcal{C}_{#1}^{#2}}

\newcommand{\pr}[1]{\mathbbm{P}\big({#1}\big)}

\newtheorem{definition}{Definition}
\newtheorem{remark}{Remark}
\newtheorem{lemma}{Lemma}

\usepackage{todonotes}
\title{Collaborative Algorithms for Online Personalized\\ Mean Estimation}

\author{\name Mahsa Asadi \email llvllahsa@gmail.com \\
      \addr Univ. Lille, Inria, CNRS, Centrale Lille\\
      UMR 9189 - CRIStAL, F-59000 Lille, France
      \ANDd
      \name Aurélien Bellet \email aurelien.bellet@inria.fr \\
      \addr Univ. Lille, Inria, CNRS, Centrale Lille\\
      UMR 9189 - CRIStAL, F-59000 Lille, France
      \ANDd
      \name Odalric-Ambrym Maillard \email odalric.maillard@inria.fr\\
      \addr Univ. Lille, Inria, CNRS, Centrale Lille\\
      UMR 9189 - CRIStAL, F-59000 Lille, France
      \ANDd
      \name Marc Tommasi
      \email marc.tommasi@inria.fr\\
      \addr Univ. Lille, Inria, CNRS, Centrale Lille\\
      UMR 9189 - CRIStAL, F-59000 Lille, France
  }

\def\month{12}  % Insert correct month for camera-ready version
\def\year{2022} % Insert correct year for camera-ready version
\def\openreview{\url{https://openreview.net/forum?id=VipljNfZSZ}} % Insert correct link to OpenReview for camera-ready version

\begin{document}

\maketitle

% !TEX root = tmlr.tex

\begin{abstract}

  We consider an online estimation problem involving a set of agents.
  Each agent has access to a (personal) process that generates samples from a
  real-valued distribution and seeks to estimate its mean. 
  We study the case where some of the distributions have the same mean, and
  the agents are allowed to actively query information from other agents. The
  goal is to design an algorithm that enables each agent
  to improve its mean estimate thanks to communication with other agents. The
  means as well as the number of distributions with same mean are unknown,
  which makes the task nontrivial. We introduce a novel collaborative strategy
  to solve this online personalized mean estimation problem. We analyze its
  time complexity and introduce variants that enjoy good performance in
  numerical experiments. We also extend our approach to the setting where
  clusters of agents with similar means seek to estimate the mean of their
  cluster.
  \end{abstract}

%%% Local Variables:
%%% mode: latex
%%% TeX-master: "supp"
%%% ispell-local-dictionary: "english"
%%% End:

% !TEX root = tmlr.tex

\section{Introduction}
\label{sec:introduction}

With the \rev{widespread} of personal digital devices, ubiquitous computing and
IoT (Internet of Things), the need for decentralized and collaborative
computing has become more \rev{pressing}.
Indeed, devices are first of all designed to collect data and this data
may be sensitive and/or too large to be transmitted.
Therefore, it is
often preferable to keep the data on-device, where it has been
collected. Local processing on a single device is a always possible option but
learning in isolation suffers from slow convergence time when data arrives
slowly. In that case, collaborative strategies can be investigated in order
to increase statistical power and accelerate learning. In recent years, such
collaborative approaches have been broadly referred to as federated learning
\citep{kairouz_advances}.

The data collected at each device reflects the specific usage, production
patterns and objective of the associated agent. Therefore, we must solve a set
of \emph{personalized tasks over heterogeneous data distributions}. Even
though the tasks are personalized, collaboration can play a significant
role in reducing the time complexity and accelerating learning in presence of
agents who share similar objectives. An
important building block to design collaborative algorithms is then to
identify agents acquiring data from the same (or similar) distribution. This
is particularly difficult to do in an \emph{online} setting, in which data
becomes available sequentially over time.

In this work, we explore this challenging objective in the context of a new
problem: \emph{online personalized mean estimation}. Formally, each agent
continuously
receives data
from a \emph{personal} $\sigma$-sub-Gaussian distribution and aims to
construct an accurate estimation of its mean as fast as possible. At each
step, each agent receives a new sample from its distribution but is also
allowed to query the current local average of another agent. To enable
collaboration, we assume the existence of an underlying
\emph{class structure} where agents in the same class have the same mean
value. We also consider a relaxed assumption where the means of agents in a
class are close (but not necessarily equal). Such assumptions are natural in
many real-world applications~\citep{adi2020machine}. A simple
example is that of in different environments, monitoring parameters such as
temperature in order to accurately estimate their mean
\citep[see for instance][]{mateo2013machine}. Another example is collaborative
filtering, where the goal is to estimate user preferences by leveraging the
existence of clusters of users with similar preferences 
\citep{10.1155/2009/421425}. Crucially, the number of classes and their cardinality are unknown to the
agents and must be discovered in an online fashion.

We propose collaborative algorithms to solve this problem, where agents
identify the class they belong to in an online fashion so as to better and
faster estimate their own mean by assigning weights to other agents' estimates.
Our approach is grounded in Probably
Approximately Correct (PAC) theory, allowing agents to iteratively
discard agents in different classes with high confidence. 
We provide a theoretical analysis of our approach by bounding the time
required by an agent to correctly estimate its class with high probability, as
well as the time required by an agent to estimate its mean to the desired
accuracy. Our results highlight the dependence on the gaps between the true
means of agents in different classes, and show that in some settings our
approach achieves nearly the same time complexity as an oracle who would know
the classes beforehand. Our numerical experiments on synthetic data are in
line with our theoretical findings and show that some empirical variants of
our approach can further improve the performance in practice.

The paper is organized as follows. Section~\ref{sec:related-work} discusses
the related work on federated learning and collaborative online learning.
In Section~\ref{sec:problem-setting}, we formally describe the problem setting
and introduce relevant notations. In Section~\ref{sec:proposed-approach}, we
introduce our algorithm and its variants.
Section~\ref{sec:theoretical-analysis} presents our theoretical analysis of
the proposed algorithm in terms of class and mean estimation time complexity.
Section~\ref{sec:numerical-results} is devoted to illustrative  numerical
experiments. Section~\ref{sec:imperfect} extends our approach to the case
where classes consist of agents with similar (but not necessarily equal)
means and agents seek to estimate the mean of their class. We conclude and
discuss perspectives for future work in Section~\ref{sec:conclusion}.
%%% Local Variables:
%%% mode: latex
%%% TeX-master: "supp"
%%% ispell-local-dictionary: "english"
%%% End:

% !TEX root = tmlr.tex

\section{Related Work}
\label{sec:related-work}

Over the last few years, collaborative estimation and learning
problems involving several agents with local datasets have been
extensively investigated under the broad term of Federated Learning (FL)
\citep{kairouz_advances}. While traditional FL algorithms learn a global
estimate for all agents, more personalized approaches have
recently attracted a lot of interest \citep[see for instance][and references
therein]
{vanhaesebrouck2017decentralized,smith2017federated,fallah2020personalized,sattler2020clustered,hanzely2020lower,Marfoq2021a}. \rev{With the exception of recent work on simple linear regression settings \citep{cheng2021federated}, these approaches typically lack clear statistical assumptions on the relation between local data distributions and do not provide error guarantees with respect to these underlying distributions. More importantly, the above methods focus on the batch learning setting and are not suitable for online learning.}

In the online setting, the work on collaborative learning has largely focused
on multi-armed bandits (MAB). Most approaches however consider a
\emph{single} MAB instance which is solved collaboratively by multiple agents.
Collaboration between agents can be implemented through broadcast
messages to all agents \citep{hillelnips13,Tao2019}, via a server 
\citep{wangiclr20}, or
relying only on local message exchanges over a network graph
\citep{sankararaman2019social,martinez2019decentralized,wang2020optimal,landgren2021distributed,Madhushani}. Other approaches do not allow explicit communication but instead consider a collision model where agents receive no reward if several agents pull the same arm \citep{boursier-perchet,wang2020optimal}. In any case, all agents aim at solving the \emph{same} task.

Some recent work considered collaborative
MAB settings where the arm means vary across agents. Extending
their previous work \citep{boursier-perchet},
\cite{boursier2019practical} consider the case where arm means can vary among
players. Under their collision model, the problem reduces to finding a
one-to-one assignment of agents to arms.
In \cite{shi2021federated_global}, the local arm means of each agent are IID
random realizations of fixed global means and the goal is to solve the global
MAB using only observations from the local arms with an algorithm inspired
from traditional FL. Similarly, \citet{Karpov22} extend the work of 
\citet{Tao2019} by considering different local arm means for each agent with
the goal to identify the arm with largest aggregated mean.
\cite{shi2021federated} introduce a limited amount of personalization by
extending the model of \cite{shi2021federated_global} to optimize a
mixture between the global and local MAB objectives. \cite{reda22} further
consider a \emph{known} weighted combination of the local MAB objectives\rev{,
and focus on the pure exploration setting (best arm identification) rather
than regret minimization.}
A crucial difference with our work is that there is no need to discover
relations between local distributions to solve the above problems.

Another related problem is to identify a graph structure on top of the
arms in MAB. \cite{kocak2020best,kocak2021best} construct a similarity graph
while solving the best arm identification problem, but consider only a single
agent. In contrast, our work considers a multi-agent setting with personalized
estimation tasks, and our approach discovers similarities across agents' tasks
in an online manner.

%%% Local Variables:
%%% mode: latex
%%% TeX-master: "supp"
%%% ispell-local-dictionary: "english"
%%% End:

% !TEX root = tmlr.tex

\section{Problem Setting}
\label{sec:problem-setting}

We consider a mean estimation problem involving $\nagent$ agents. The goal of
each agent $\oneagentindex \in [\nagent] = \{1,2, \dots,\nagent\}$ is to
estimate the mean $\mu_a$ of a personal distribution $\rv{\nu}{\oneagentindex}
{}$ over $\mathbb{R}$. In this work, we assume that there exists
$\sigma\geq 0$ such that each $\rv{\nu}{\oneagentindex}{}$ is
$\sigma$-sub-Gaussian, i.e.:
\begin{equation*}
\forall \lambda\in\mathbb{R},\quad  \log\mathbb{E}_{x\sim\nu_a} \exp( \lambda 
(x-\mu_a)) \leq \frac{1}{2}\lambda^2\sigma^2.
\end{equation*}
This classical assumption captures a property of strong tail decay, and includes in particular
Gaussian distributions (in that case, the smallest possible $\sigma^2$
corresponds to the variance) as well as any distribution supported on a
bounded interval (e.g., Bernoulli distributions).

We consider an online and collaborative setting where data points are received sequentially and agents can query each other to learn about their respective distributions.
Agents should be thought of as different user devices which operate in
parallel. Therefore, they all receive a new sample and query another agent at
each time step.

Formally, we assume that time is synchronized between agents and at each time
step $t$, each agent $a$ receives a new sample $x_a^t$ from its personal
distribution $\nu_a$ with mean $\rv{\mu}{\oneagentindex}{}{}$, which is used
to update its local mean estimate $\rv{\bar{x}}  
{\oneagentindex,\oneagentindex}{t}=\frac{1}{t}\sum_{t'=1}^t x_a^{t'}$. It also chooses another agent $\onememoryindex$ to \emph{query}. As a response from querying agent $\onememoryindex$, agent $a$ receives the local average $ \rv{\bar{x}}{\onememoryindex,\onememoryindex}{t}$ of agent $\onememoryindex$ (i.e., the average of $t$ independent samples from the personal distribution $\nu_l$) and stores it in its memory $\rv{\bar{x}} {\oneagentindex, \onememoryindex}{t}$ along with the corresponding number of samples $n_{a, l}^t=t$. Each agent $\oneagentindex$ thus keeps a memory $[(\rv{\bar{x}}{\oneagentindex, 1}{t},n_{a,1}^t), \dots, (\rv{\bar{x}} {\oneagentindex, \nagent}{t},n_{a,A}^t)]$ of the last local averages (and associated number of samples) that it received from other agents. The information contained in this memory is used to compute an \emph{estimate} $\mu_a^t$ of $\mu_a$ at each time $t$. Our goal is to design a query and estimation procedure for each agent.

As described above, note that when an agent queries another agent at time $t$,
it does not receive one sample from this agent (as e.g. in multi-armed
bandits), but receives the full statistics of observations of this agent up to
time $t$. This is considerably much more information than in typical MAB
settings, and naturally requires  specific strategies.

\rev{The goal of each agent to find
a good estimate of its personal mean as fast as possible, without
consideration for the quality of the estimates in earlier steps (i.e., we do
not seek to minimize a notion of regret).
In the online learning terminology, this is referred to as a pure exploration
setting
\citep[like best arm identification in multi-armed bandits, see][]{bai}.}
Formally, we will measure the performance of an algorithm using the
notion of $(\epsilon, \delta)$-convergence in probability 
\citep{bertsekas2002introduction,wasserman2013all}, which we recall below.

\begin{definition}[PAC-convergence]
\label{def:pac}
An estimation procedure for agent
$\oneagentindex$ is called $(\epsilon, \delta)$-convergent if there exists
$\tau_a\in\mathbb N$ such that the probability that the mean estimator $\mu_
{\oneagentindex}^t$ of agent $\oneagentindex$ is $\epsilon$-distant from the
true mean for any time   $t > \stoppingtime{\oneagentindex}$ is at least $1 -
\delta$:
	\begin{equation}
	\pr
	{
		\forall t >\stoppingtime{\oneagentindex}
		:
		|
		\mu_\oneagentindex^t
		-
		\mu_{\oneagentindex}
		|
		\leq
		\epsilon
	}
	>
	1 - \delta.
	\end{equation}
\end{definition}

While it is easy to design $(\epsilon, \delta)$-convergent estimation
procedure for a single agent taken in isolation, the goal of this paper is to
propose collaborative algorithms where agents benefit from information from
other agents by taking advantage of the relation between the agents'
distributions. This will allow them to build up more accurate estimations in
less time, i.e., with smaller time complexity $\stoppingtime{\oneagentindex}$.

Specifically, to foster collaboration between agents, we consider that the set
of agents $[A]$ is partitioned into equivalence classes that correspond to
agents with the same
mean.\footnote{In Section~\ref{sec:imperfect}, we will consider a relaxed
version of this assumption where classes consist of agents with \emph{similar}
(not necessarily equal) means.}
In real
scenarios, these classes may represent sensors in the same environment,
objects with
the same technical characteristics, users with the same behavior, etc. This
assumption makes it possible for an agent to design strategies to identify
other agents in the same class and to use their estimates in order to speed up
the estimation of his/her own mean.
Formally, we define the class of $a$ as the set of agents who have the same
mean as $a$.

\begin{definition}[Similarity class]
	The \emph{similarity class} of agent $\oneagentindex$ is given by:
	\begin{align*}
	\class{\oneagentindex} 
	= 
	\{
	\onememoryindex \in [\nagent]: 
	\Delta_{\oneagentindex, \onememoryindex} = 0
	\},
	\end{align*}
	where $\Delta_{\oneagentindex,\onememoryindex}
	= 
	|
	\mu_\oneagentindex
	-
	\mu_\onememoryindex
	|$ is the gap between the means of agent $a$ and agent $l$. 
	\label{def:true_similarity}
\end{definition}

The gaps $\{\Delta_{\oneagentindex,\onememoryindex}\}_
{\oneagentindex,\onememoryindex\in [A]}$ define the problem structure. We
consider that the agents do not know the means, the gaps, or even the number
of underlying classes. Hence the classes
$\{\class{\oneagentindex}\}_{\oneagentindex\in [A]}$ are completely unknown.
This makes the problem quite challenging.

\rev{
\begin{remark}[Scalability]
\label{rem:scale}
In large-scale systems, it may be impractical for agents to query all other
agents and/or to maintain a memory size that is linear in the number of agents
$A$.
From a practical point of view, each agent can instead consider a 
restricted subset of agents of reasonable size. This subset could be picked
uniformly at random, be composed of neighboring nodes in the network or in
the physical world, or be based on prior knowledge on who is
more likely to be in the same class (when available).
\end{remark}
}

%%% Local Variables:
%%% mode: latex
%%% TeX-master: "tmlr"
%%% ispell-local-dictionary: "english"
%%% End:

% !TEX root = tmlr.tex

\section{Proposed Approach}
\label{sec:proposed-approach}

In this section, we first introduce some of the key technical components used in
our approach, and then present our proposed algorithm.

\subsection{Main Concepts}
In our approach, each agent $a$ computes confidence
intervals $I_{\oneagentindex, \onememoryindex} = [\rv{\bar{x}}{\oneagentindex, \onememoryindex}{t}-\CB{\delta}{\nobservation{\oneagentindex, \onememoryindex}{t}}, \rv{\bar{x}}{\oneagentindex, \onememoryindex}{t} + \CB{\delta}{
\nobservation{\oneagentindex, \onememoryindex}{t}}]$ for the mean estimators $[\rv{\bar{x}}{\oneagentindex, 1}{t}, \dots, \rv{\bar{x}}{\oneagentindex, \nagent}{t}]$
that it holds in memory at time $t$. The generic confidence bound $\CB{\delta}
{
\nobservation{\oneagentindex, \onememoryindex}{t}}$ takes as input the number
of samples $\nobservation{\oneagentindex, \onememoryindex}{t}$ seen for
agent $l$ at time $t$, and $\delta$ corresponds to the risk
parameter so that with probability at least $1-\delta$ the true mean $\mu_l$
falls within the confidence interval $I_{\oneagentindex, \onememoryindex}$. 

Agent $a$ will use these confidence intervals to assess whether another agent
$l$ belongs to the class $\class{\oneagentindex}$ through the evaluation of an optimistic
 distance defined below.

\begin{definition}[Optimistic distance]
  The \emph{optimistic distance} with agent $\onememoryindex$
  from the perspective of agent $\oneagentindex$ is defined as:
  \begin{align}
    \distance{\oneagentindex,\delta}{t}{\onememoryindex}
    =
    |
    \rv{\bar{x}}{\oneagentindex, \oneagentindex}{t} 
    -
    \rv{\bar{x}}{\oneagentindex, \onememoryindex}{t}
    |
    -
    \CB{\errorprob}{\nobservation{\oneagentindex, \oneagentindex}{t}}
    -
    \CB{\errorprob}{\nobservation{\oneagentindex, \onememoryindex}{t}}.
    \label{def:optimistic_similarity}
  \end{align}
\end{definition}
The ``optimistic'' terminology is justified by the fact that $\distance{\oneagentindex,\delta}{t}{\onememoryindex}$ is,
with high probability, a \emph{lower bound} on the distance between the true
means $\mu_\oneagentindex$ and $\mu_\onememoryindex$ of distributions
$\nu_\oneagentindex$ and $\nu_\onememoryindex$.
Recall that two agents belong to the same class if the distance of their true
mean is zero. Since these
values are unknown, the idea of the above heuristic is to provide a proxy
based on observed data and high probability confidence bounds. In particular,
we will adopt the \textit{Optimism in Face of
Uncertainty Principle (OFU)} \citep[see][]{auer2002finite} and consider that
two agents may be in the same class if the optimistic
distance is zero or less. Hence, we define an optimistic notion of class
accordingly.

\begin{definition} 
  \label{def:class_heuristic}
  The \emph{optimistic similarity class}  from the perspective of agent
  $\oneagentindex$ at time $t$ is defined as:
  \begin{align*}
    \classt{\oneagentindex}{t}
    =
    \{
    \onememoryindex \in [\nagent]
    :
    \distance{\oneagentindex, \delta}{t}{\onememoryindex} \leq 0 
    % \land
    % \distance{\oneagentindex, +}{t}{\oneagentindex}{\onememoryindex} \leq \similaritythreshold
    \}.
  \end{align*}
\end{definition}

Having introduced the above concepts, we can now present our algorithm.

\subsection{Algorithm}
\begin{algorithm}[tb] % enter the algorithm environment
	\caption{\algo} 
    \begin{algorithmic}[0] % enter the algorithmic 
    \renewcommand{\algorithmicrequire}{\textbf{Parameters:}}
    \REQUIRE agent $\oneagentindex$, time horizon $H$, risk $\delta$, weighting scheme
    $\alpha$, and query strategy  \Choose
	% \STATE $[\rv{\bar{x}}{\oneagentindex, 1}{0}, \dots, \rv{\bar{x}}
 %    {\oneagentindex, \nagent}{0}] \leftarrow [0,\dots,0]$
 %    \STATE $[
 %    \nobservation{\oneagentindex, 1}{0}, \dots, \nobservation{\oneagentindex, \nagent}{0}] \leftarrow [0,\dots,0]$
  \STATE $\forall l\in[A]$: $\rv{\bar{x}}{\oneagentindex, l}{0}\leftarrow 0$,
  $\nobservation{\oneagentindex, l}{0} \leftarrow 0$
    \STATE 	$ \class{\oneagentindex}^0 \leftarrow \{\onememoryindex \in [\nagent]:
		\distance{\oneagentindex, \delta}{t}{\onememoryindex} \leq 0 \}=[\nagent]$
		\FOR{$t = 1, \dots, H$}
  \STATE $\forall l\in[A]$: $\rv{\bar{x}}{\oneagentindex, l}{t}\leftarrow \rv{
  \bar{x}}{\oneagentindex, l}{t-1}$,
  $\nobservation{\oneagentindex, l}{t} \leftarrow 
  \nobservation{\oneagentindex, l}{t-1}$
  			\STATE \textbf{Perceive:}\nonumber
			\STATE $\quad$ Receive sample $x_a^t\sim\nu_a$
			\STATE $\quad$ $\rv{\bar{x}}{\oneagentindex, \oneagentindex}{t}
           \leftarrow 
            \rv{\bar{x}}{\oneagentindex, \oneagentindex}{t-1}  \times \frac{t - 1}{t} + x_a^t \times \frac{1}{t}$, ~$\nobservation{\oneagentindex, \oneagentindex}{t} \leftarrow t$
		\STATE \textbf{Query:}
		\STATE $\quad$
		$ \class{\oneagentindex}^t \leftarrow \{\onememoryindex \in [\nagent]:
		\distance{\oneagentindex, \delta}{t}{\onememoryindex} \leq 0 \}$
		\STATE $\quad$ Query agent $l = \Choose(
    \class{\oneagentindex}^t)$ to get $\rv{\bar{x}}{l, l}{t}$
		\STATE $\quad$ $
		\rv{\bar{x}}{\oneagentindex, l}{t} \leftarrow
		\rv{\bar{x}}{l, l}{t}
		,$ ~$\nobservation{\oneagentindex, l}{t} \leftarrow t$
		\STATE \textbf{Estimate:}
		\STATE $\quad$
		$ \class{\oneagentindex}^t \leftarrow \{\onememoryindex \in [\nagent]:
		\distance{\oneagentindex, \delta}{t}{\onememoryindex} \leq 0 \}$\label{algo:class_estimation}
		\STATE $\quad$ $\mu_{\oneagentindex}^t \leftarrow \sum_
        {\onememoryindex \in
        \class{\oneagentindex}^t}  \weighting{\oneagentindex, \onememoryindex}{t} \times \rv{\bar{x}}{\oneagentindex, \onememoryindex}{t}$
		\label{alg:weighted-estimation}
	\ENDFOR
	\end{algorithmic}
	\textbf{Output:} $\mu_\oneagentindex^H$ % $\rv{\bar{x}}{\oneagentindex}{H}$
	\label{algorithm:main}
\end{algorithm}

The collaborative mean estimation algorithm we propose, called \algo, is given
in Algorithm~\ref{algorithm:main} (taking the 
perspective of agent $a$).
For conciseness, we consider that $\CB{\delta}{0}=+\infty$.  At each step $t$,
agent $a$ performs three main~steps.

In the \textbf{Perceive} step, the agent receives a sample from its distribution and updates its local average together with the number of samples.

In the \textbf{Query} step, agent $a$ selects another agent following a query
strategy given as a parameter to the \algo algorithm. Agent $a$ runs the \Choose function to select another agent $l$  and asks for its current local estimate to update its memory.
We propose two variants for the \Choose function:
\begin{itemize}
\item \RR: cycle over the set $[A]$ of agents one by one in a fixed
order.
\item \RRR: like round-robin but ignores agents that are not in the set of optimistically similar agents $\classt{\oneagentindex}{t}$.
\end{itemize}

The focus on round-robin-style strategies is justified by the information
structure of our problem setting, which is very different from classic bandits. Indeed, querying an agent at time $t$ produces an estimate computed on the
$t$ observations collected by this agent so far.
The choice of variant (\RR or \RRR) will affect the class
identification time complexity, as we shall discuss later.

Finally, in the \textbf{Estimate} step, agent $a$ computes the optimistic
similarity
class $\classt{\oneagentindex}{t}$ based on available information, and
constructs its mean estimate as a weighted aggregation of the local averages of agents that belong to $\classt{\oneagentindex}{t}$. We propose different weighting mechanisms:

\paragraph{Simple weighting.} This is a natural weighting mechanism for
aggregating samples:
\begin{equation*}
  \weighting{\oneagentindex, \onememoryindex}{t}
  =
  \frac
  {
    \nobservation{\oneagentindex, \onememoryindex}{t}
  }
  {
    \sum_{\onememoryindex \in C_{a}^t}
    \nobservation{\oneagentindex, \onememoryindex}{t}
  }.
\end{equation*}
\paragraph{Soft weighting.} This is a heuristic weighting mechanism which
leverages the intuition that the more the
confidence intervals of two agents overlap, the more likely that they are in
the
same class. Moreover, the smaller the union of the agent means, the more
confident we are that the agents are in the same class. In other words, we are
not equally confident about all the agents that are selected for estimation,
and this weighting mechanism incorporates this information:
\begin{equation*}
  \weighting{\oneagentindex, \onememoryindex}{t} =
  \nobservation{\oneagentindex, \onememoryindex}{t}
  \frac{
    |I_{\oneagentindex, \oneagentindex} \cap I_{\oneagentindex, \onememoryindex}|
  }
  {
    |I_{\oneagentindex, \oneagentindex} \cup I_{\oneagentindex, \onememoryindex}|
  }
  \times
  \frac
  {
    1
  }{ Z_{\mathrm{soft}}
  },
\end{equation*}
where $Z_{\mathrm{soft}}=\sum_{i \in C_{a}^t}   \frac{ \nobservation{\oneagentindex, i}{t} |I_{\oneagentindex, \oneagentindex} \cup I_{\oneagentindex, i}| }{ |I_{\oneagentindex, \oneagentindex} \cap I_{\oneagentindex, i}| }$ is a normalization factor.

\paragraph{Aggressive weighting.} This is an extension of the previous soft
weighting mechanism that is more selective. Not only
does it consider the overlap and intersection of the agents' confidence
intervals, but it also requires the size of the intersection to be larger than
half the size of both confidence intervals from the two agents.
Let us denote
the binary value associated with this condition by $E_{a,l} = \indicator{|I_{\oneagentindex, \oneagentindex} \cap I_{\oneagentindex, \onememoryindex}|
      > \min\{\CB{\delta}{\nobservation{\oneagentindex, \onememoryindex}{t}}
      ,
      \CB{\delta}{\nobservation{\oneagentindex, \oneagentindex}{t}}\}}$. Then
\begin{equation*}
  \weighting{\oneagentindex, \onememoryindex}{t} =
  \nobservation{\oneagentindex, \onememoryindex}{t}
  \frac{
    |I_{\oneagentindex, \oneagentindex} \cap I_{\oneagentindex, \onememoryindex}|
  }
  {
    |I_{\oneagentindex, \oneagentindex} \cup I_{\oneagentindex, \onememoryindex}|
  }
  \times
  \frac
  {
    E_{a,l}      
  }{
    Z_{\mathrm{agg}}
  },
\end{equation*}
where $Z_{\mathrm{agg}}=\sum_{i \in C_{a}^t}
      \frac{
        \nobservation{\oneagentindex, i}{t}
        |I_{\oneagentindex, \oneagentindex} \cup I_{\oneagentindex, i}|
        \times
        E_{a,i}
      }{
        |I_{\oneagentindex, \oneagentindex} \cap I_{\oneagentindex, i}|}$ is a normalization factor.

\subsection{Baselines}
\label{sec:baselines}

We introduce two baselines that will be used to put the performance of our approach into
perspective, both theoretically and empirically.

\paragraph{Local estimation.} Estimates are computed without any collaboration,
using only
samples received from the agent's own distribution, i.e. $\rv{\mu}
{\oneagentindex}{t}
= \rv{\bar{x}}{\oneagentindex, \oneagentindex}{t}$.

\paragraph{Oracle weighting.} The agent knows the true class $\class{\oneagentindex}$ via an oracle and uses the simple weighting $ 
    \weighting{\oneagentindex, \onememoryindex}{t} = 
    \frac{\nobservation{\oneagentindex, \onememoryindex}{t}}{\sum_{\onememoryindex \in \class{\oneagentindex}}
    \nobservation{\oneagentindex, \onememoryindex}{t}
    }$.

%%% Local Variables:
%%% mode: latex
%%% TeX-master: "tmlr"
%%% ispell-local-dictionary: "english"
%%% End:

% !TEX root = tmlr.tex

\section{Theoretical Analysis}
\label{sec:theoretical-analysis}

In this section, we provide a theoretical analysis of our algorithm \algo  for
the query strategy \RRR  and the simple weighting scheme. Specifically, we bound the time
complexity  in probability for both class and mean
estimation. Below, we explain the key steps involved in our analysis and
state our main results. All proofs can be found in the appendix.

A key aspect of our analysis is to characterize when
% required for the analysis is the number of samples needed
% to reach the guarantee that
the optimistic similarity class  (Definition~\ref{def:class_heuristic})
coincides with the true classes. We show that this is the case when two
conditions hold. First, for a given agent $a$, we need the confidence
interval computed by $a$ about agent $l$ to contain the true mean $\mu_l$ for
all $l\in A$.
\begin{definition}
	%	Using optimistic similarity heuristic, for any communication strategy, given $\delta'$ as a function of input confidence $\delta$, given $\epsilon$ as the mean estimation error threshold and under event, 
  \label{def:event_bounds}
	We define the following events:
	\begin{gather}
		E_{\oneagentindex}^{t}
		=
		\underset{ \onememoryindex \in [\nagent]}{\bigcap}
		|
		\rv{\bar{x}}{\oneagentindex, \onememoryindex}{t}
		-
		\rv{\mu}{\onememoryindex}{}
		|
		\leq
		\CB{\delta}{\nobservation{\oneagentindex, \onememoryindex}{t}},\\
                E_{\oneagentindex}
		=
		\underset{ t \in \mathbbm{N}}{\bigcap}
		E_{\oneagentindex}^t.
	\end{gather}
	%	and $\CB{\delta}{\nobservation{\oneagentindex}{t}}$ is a generic form of a confidence bound.\\
\end{definition}

We can guarantee that $E_{\oneagentindex}$ holds with high probability
via an appropriate parameterization of confidence intervals. We use the
so-called Laplace method \citep{maillard2019mathematics}.

\begin{restatable}{lemma}{lemmaone}
\label{lemma:E_a_true}
  Let $\delta\in(0,1)$, $a\in[A]$.
  % By choosing time-uniform
  % $\sigma$-sub-Gaussian
  % concentration method, 
  Setting $\CB{\delta}{n} = 
  % \sum_{s=1}^{t}
  \sigma\sqrt
  {
    2\frac{1}{n}\times
    (
    1
    + 
    \frac{1}{n}
    )
    \ln(
    \sqrt
    {
      n + 1
    }
    /
    \gamma(\delta)
    )
  }
  $ with $\gamma(\delta) = \frac{\delta}{8\times \nagent}$, we have:
  \begin{align}
  \pr{E_{\oneagentindex}} \geq 1 - \frac{\delta}{8}.
  \end{align}
\end{restatable}

The second condition is that agent's $a$ memory about the local estimates of other agents should contain enough samples. Let us denote by $\lceil\CBinv{\delta}{x}\rceil$ the smallest integer $n$ such that $x>\CB{\delta}{n}$. 
\begin{definition}
  From the perspective of agent $\oneagentindex$ and at time $t$, event $G_{\oneagentindex}^t$ is defined as:
  \begin{align}
    G_{\oneagentindex}^t
    =
    \underset{\onememoryindex \in [\nagent]}{\bigcap}
    \nobservation{\oneagentindex, \onememoryindex}{t}
    >
    \nobservation{\oneagentindex, \onememoryindex}{\star},
  \end{align}
  \begin{equation*}
    \text{where ~~~~~~~} \nobservation{\oneagentindex, \onememoryindex}
    {\star}
    = 
    \begin{cases}
      \lceil \CBinv{\delta}{\frac{\Delta_{\oneagentindex, \onememoryindex}}{4}}\rceil
      & \text { if }\onememoryindex \notin \class{\oneagentindex},
      \\
      \lceil \CBinv{\delta}{\frac{\Delta_{\oneagentindex}}{4}}\rceil
      & \text{ otherwise},
  \end{cases}
  \end{equation*}
  with  $\Delta_{\oneagentindex} = \min_{\onememoryindex \in 
  [\nagent]\backslash\class{\oneagentindex}} \Delta_{\oneagentindex, \onememoryindex}$.
  \label{def:needed_samples}
\end{definition}

Note that the required number of samples is inversely proportional to the
gaps between the means of agents in different classes. Having enough
samples and knowing that the true means fall within the
confidence bounds, we can show that the class-estimation rule $
\distance{\oneagentindex,\delta}{t}{\onememoryindex}\leq 0$ indicates the
membership of
$\onememoryindex$ in $\class{\oneagentindex}$.
\begin{restatable}[Class membership rule]{lemma}{lemmatwo}
  Under $E_{\oneagentindex}^{t} \land G_{\oneagentindex}^t$ and
  $\forall \onememoryindex \in [\nagent]$ and at time $t$:
  $\distance{\oneagentindex,\delta}{t}{\onememoryindex} > 0 \iff
  \onememoryindex \in [\nagent]\backslash\class{\oneagentindex}$.
  \label{lemma:class-membership-rule}
\end{restatable}

Using the above lemma, we obtain the following result for the time complexity
of class estimation.

\begin{restatable}[\algo class estimation time complexity]{theorem}{thmone}
  \label{theorem:class_estimation}
  For any $\delta \in (0,1)$, employing \RRR query strategy, we have:
  \begin{equation}
    \pr{
      \exists t >\switchingtime{\oneagentindex}: \classt{\oneagentindex}{t}
      \neq \class{\oneagentindex}
    }
    \leq
    \frac{\delta}{8},
  % the probability of error is bounded by  $\frac{\delta}{4}$ where the class estimation time complexity for agent $\oneagentindex$ is defined as:
  \quad\text{with ~}
    \switchingtime{\oneagentindex}
    =
    \nobservation{\oneagentindex, \oneagentindex}{\star}
    +
    \nagent -1 
    -
    \sum_{\onememoryindex \in [\nagent]\setminus\class{\oneagentindex}}
    \indicator
    {\nobservation{\oneagentindex, \oneagentindex}{\star}
      >
      \nobservation{\oneagentindex, \onememoryindex}{\star} 
      +
      \nagent - 1 
    }.
\label{eq:switchingtime}
  \end{equation}
\end{restatable}

\rev{The dominating term in the class
estimation time complexity $
\switchingtime{\oneagentindex}$ for
agent $\oneagentindex$ is equal to the
number $\nobservation{\oneagentindex, \oneagentindex}{\star}$ of samples required to distinguish agent $a$ from the one who has
smallest nonzero gap $\Delta_a$ to $a$, which is of order $\widetilde{O}
(1/\Delta_a^2)$.}\footnote{\rev{We use $\widetilde{O}(\cdot)$ to hide
constant and logarithmic terms.}} There is then an additional term of
$\nagent
- 1$ since all others agents that are not in $\class{\oneagentindex}$ could require
the same number of samples. Finally, the last term in \eqref{eq:switchingtime}
accounts
for agents that require less samples and had thus been eliminated before,
which reflects the gain of using \RRR query strategy over \RR.
When we have enough samples (at least $\switchingtime{\oneagentindex}$),
Theorem~\ref{theorem:class_estimation} guarantees that we correctly learn the
class ($\class{\oneagentindex} = \classt{\oneagentindex}{t}$) with high
probability. We build upon this result to quantify the mean estimation time
complexity of our approach.

\begin{restatable}[\algo mean estimation time complexity]{theorem}{thmtwo}
  \label{thm:mean_estimation}
  Given the risk parameter $\delta$, using the \RRR query strategy and simple
  weighting, the mean estimator $\rv{\mu}{\oneagentindex}{t}$ of agent
  $\oneagentindex$ is $(\epsilon,\frac{\delta}{4})$-convergent, that is:
  % converges to being $\epsilon-$close of the true mean $\rv{\mu}{\oneagentindex}{}$ with high probability $1 - \frac{\delta}{4}$:
  \begin{align}
  \mathbbm{P}
  \big(
  \forall t > \stoppingtime{\oneagentindex}
  :
  |\rv{\mu}{\oneagentindex}{t}
  -
  \rv{\mu}{ \oneagentindex}{}|
  \leq
  \epsilon
  \big)
  >
  1 - \frac{\delta}{4},
\quad
  \text{with ~} \stoppingtime{\oneagentindex}
  =
  \max
  (
  \switchingtime{\oneagentindex}
  ,
  \frac
  {\lceil\CBinv{\delta}{\epsilon}\rceil}{|\class{\oneagentindex}|}
  +\frac{|\class{\oneagentindex}| - 1}{2}
  ).
  \label{eq:tau_a}
  \end{align}
\end{restatable}

Several comments are in order. First, recall that collaboration induces a
bias in mean estimation before class estimation time. Because the problem structure is unknown, any collaborative algorithm that aggregate observations from different agents will suffer from this bias, but the bias vanishes as soon as the class is estimated and we outperform local estimation.

Then, to interpret the guarantees provided by
Theorem~\ref{thm:mean_estimation}, it
is useful to compare them with the local estimation baseline, which has time
complexity $\lceil\CBinv{\delta}{\epsilon}\rceil\rev{=\widetilde{O}
(1/\epsilon^2)}$. Inspecting \eqref{eq:tau_a}, we see
that our approach has a time complexity of $\rev{\widetilde{O}
(\max\{1/\Delta_a^2,1/\epsilon^2|\class{\oneagentindex}|\})}$. In other words,
it is faster than local estimation as long as the time $
\switchingtime{\oneagentindex}$
needed to correctly \rev{identify} the class $
\class{\oneagentindex}$ is smaller than $\lceil\CBinv
{\delta}{\epsilon}\rceil$, that is precisely when:
\begin{align}
  \epsilon
  <
  \CB{\delta}
  {
    \nobservation{\oneagentindex, \oneagentindex}{\star}
    +
    \nagent
    -1
    -
    \sum_{\onememoryindex \in [\nagent]\backslash\{\class{\oneagentindex}\}}
    \indicator{\nobservation{\oneagentindex, \oneagentindex}{\star}
    >
    \nobservation{\oneagentindex, \onememoryindex}{\star}
    +
    \nagent
    -
    1
    }
  }.
  \label{eq:epsilon_threshold}
\end{align}
This condition\rev{, which roughly amounts to $\epsilon < \Delta_a$,} relates
the
desired precision of the solution $\epsilon$
to the problem structure captured by the gaps $
\{\Delta_{\oneagentindex,\onememoryindex}\}_{\onememoryindex\in [A]}$ between the true
means through $\{\nobservation{\oneagentindex, \onememoryindex}{\star}\}_
{\onememoryindex\in [A]}$ (see Definition~\ref{def:needed_samples}).
We will see in our experiments that our theory predicts quite well whether an
agent empirically benefits from collaboration.

Remarkably, \rev{our approach can be up to $|\class{\oneagentindex}|$ times
faster
than local estimation: this happens roughly when $\epsilon < \Delta_a / 
\sqrt{|\class{\oneagentindex}|}$, i.e.,} for
large enough gaps or small enough $\epsilon$. In this regime, \emph{the
speed-up achieved by
our approach is nearly optimal}. Indeed, the time complexity of the oracle
weighting baseline introduced in Section~\ref{sec:baselines} is precisely $\frac
  {\lceil\CBinv{\delta}{\epsilon}\rceil}{|\class{\oneagentindex}|}
  +\frac{|\class{\oneagentindex}| - 1}{2}\rev{=\widetilde{O}(1/\epsilon^2|
  \class{\oneagentindex}|)}$, just like our approach. Note that in a full
  information
  setting where agent $a$ would know $\class{\oneagentindex}$ \emph{and} would
also  have
  access
  to up-to-date samples from all agents at each step, the time complexity
  would be only slightly smaller, namely $\frac{\lceil\CBinv{\delta}
  {\epsilon}\rceil}
{|\class{\oneagentindex}|}$.

\rev{
\begin{remark}[Frequency of communication]
\label{rem:sample_vs_com}
For simplicity, we consider that agents communicate each
time they collect a new sample, which is standard in the literature of
collaborative learning (see for instance the
collaborative MAB approaches discussed in Section~\ref{sec:related-work}).
However,
different trade-offs between communication and data collection can be
considered. In particular, it is straightforward to adapt the
setting and our results to the case where each agent collects $m$ samples
between each communication: it amounts to multiplying by $m$ the number of
observations in our confidence intervals and empirical estimates. 
This provides a way to reduce communication, as well as to mitigate
privacy concerns by ensuring that only sufficiently aggregated
quantities are exchanged (even in early rounds).
Extensions to cases where the number of samples between each communication is
random and/or varies across agents are left for future work.
\end{remark}
}

%%% Local Variables:
%%% mode: latex
%%% TeX-master: "supp"
%%% ispell-local-dictionary: "english"
%%% End:

% !TEX root = tmlr.tex

\section{Numerical Results}
\label{sec:numerical-results}
In this section, we provide numerical experiments on synthetic data to
illustrate our theoretical results and assess the practical performance of our
proposed algorithms.\footnote{The code can be found at \href{https://github.com/llvllahsa/CollaborativePersonalizedMeanEstimation}{https://github.com/llvllahsa/CollaborativePersonalizedMeanEstimation}}

\subsection{Experimental Setting}
We consider $\nagent = 200$ agents, a time
horizon of
$2500$
 steps and % over $20$ runs,
 a risk parameter $\delta=0.001$.
 The personal distributions of agents are all Gaussian with variance
 $\sigma^2=0.25$ and belong to one of 3 classes with
means $0.2$, $0.4$ and $0.8$. The class membership of each agent (and thus
the value of its true mean) is chosen uniformly at random among the three
classes. We thus
obtain roughly balanced class sizes. While the evaluation presented in
this section
focuses on this 3-class problem, in the appendix we provide
additional results on a simpler 2-class problem (with means $0.2$ and $0.8$)
where the benefits of our algorithm is even more significant.

We consider several variants of our algorithm \algo: we compare query
strategies \RR and \RRR with simple weighting, and also evaluate the use of
soft and aggressive weighting schemes in the \RRR case.  This gives 4
 variants of our algorithm: \RR, \RRR, \SRRR and \ARRR.

Regarding competing approaches, we recall that our setting is novel
and we are not aware of existing algorithms addressing the same problem. We
can however compare against two baseline strategies.
The \Local baseline corresponds to the case of no collaboration. On the
other hand, the \Oracle
baseline represents an upper bound on the achievable performance by any
collaborative algorithm as it is given as input the true class membership of
each agent and thus does not need to perform class estimation.

 All algorithms are compared across 20 random runs
 corresponding to 20 different samples. In
a given run, at each time step, each agent receives the same sample for all
algorithms.

\subsection{Class Estimation}
\label{sec:class-identification}

We start by investigating the performance in class estimation.
In this experiment, only \RR and \RRR are shown since the different weighting
schemes have no effect on class estimation.

We first look at how well an agent $a$ estimates its true class $
\class{\oneagentindex}$ with its heuristic class $
\classt{\oneagentindex}
{t}$ across time. To measure this, we
consider the precision at time $t$ computed as follows:
\begin{align}
	\text{precision}_{\classt{\oneagentindex}{t}} = \frac{|
    \classt{\oneagentindex}{t} \cap  \class{\oneagentindex}|}{|\classt{\oneagentindex}{t}|}.
    \label{eq:precision}
\end{align}
We compute the average and standard deviation of \eqref{eq:precision} across
runs, and then average these over all agents.
Figure~\ref{fig:precision3} shows how the precision of class
estimation varies
across time as agents progressively remove others from their heuristic
class and eventually identify their true class. 
% steps at which two
% classes are distinguished over the network.
\rev{As can be seen clearly in Figure~\ref{fig:identification} in the
appendix,} the
classes $0.2$ and $0.8$ are separated very early, quickly followed by $0.4$
and $0.8$ and finally, after sufficiently many samples have been collected,
the pair with the smallest gap ($0.4$ and $0.2$).

We also observe that
\RR and \RRR only differ slightly in the last time steps before classes are
identified.
\rev{
  This is in line with Theorem~\ref{theorem:class_estimation}, which shows
  that class estimation time mainly depends on $n^{\star}_{a,a}$, the time
  needed to eliminate the agent with smallest nonzero gap. This dominant term
  and the fact that querying an agent at time $t$ yields the full statistics of observations of this agent up to time $t$ explain that the gain of \RRR is small compared to vanilla \RR.
}

\begin{figure*}[t]
    \centering
	\subfigure[Class estimation precision over all agents]{\includegraphics
    [width=.38\textwidth]{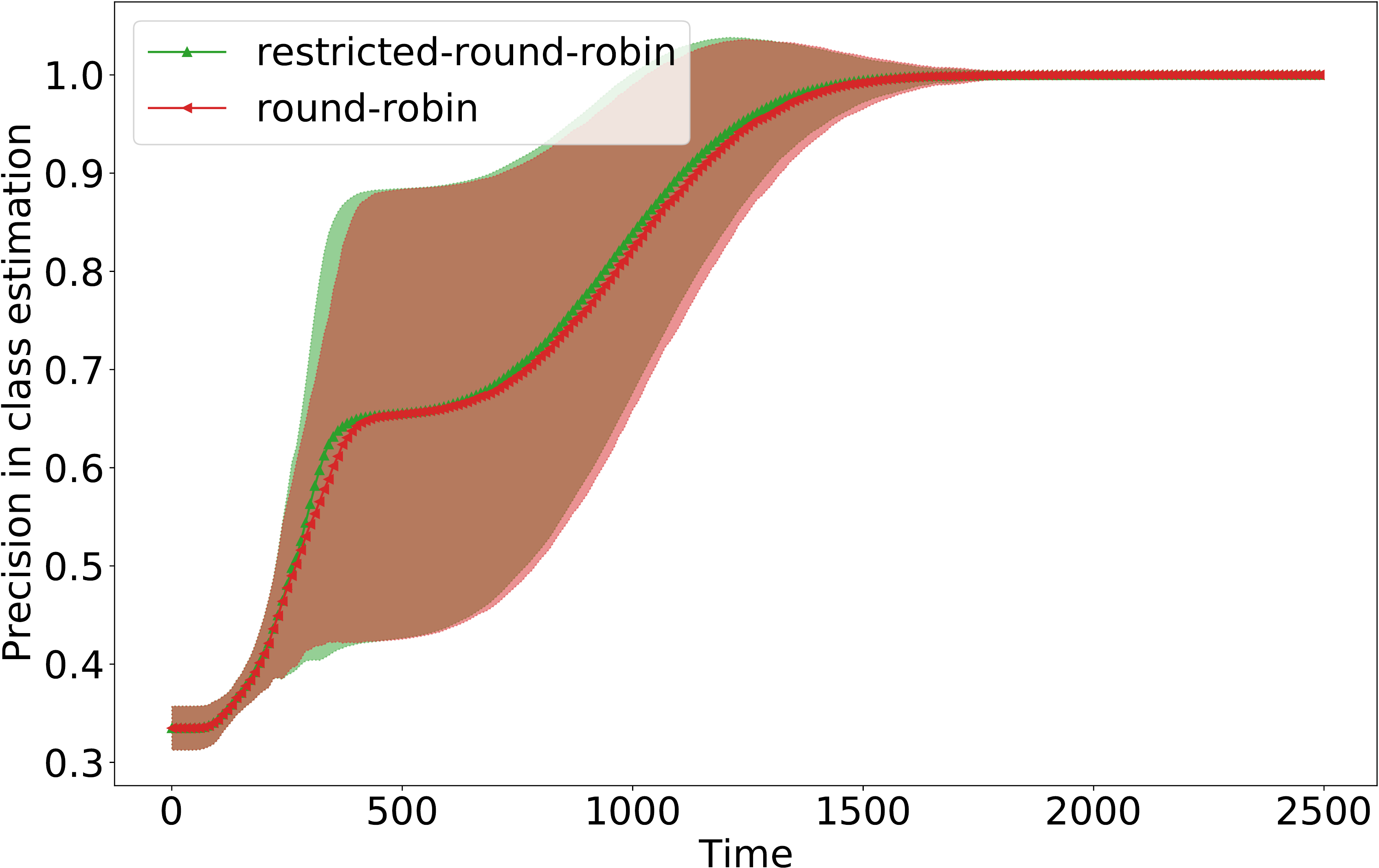} 
    \label{fig:precision3}}
    \hspace*{.05\textwidth}
	\subfigure[Error in mean estimation over all agents]
    {\includegraphics
	[width=.4\textwidth]{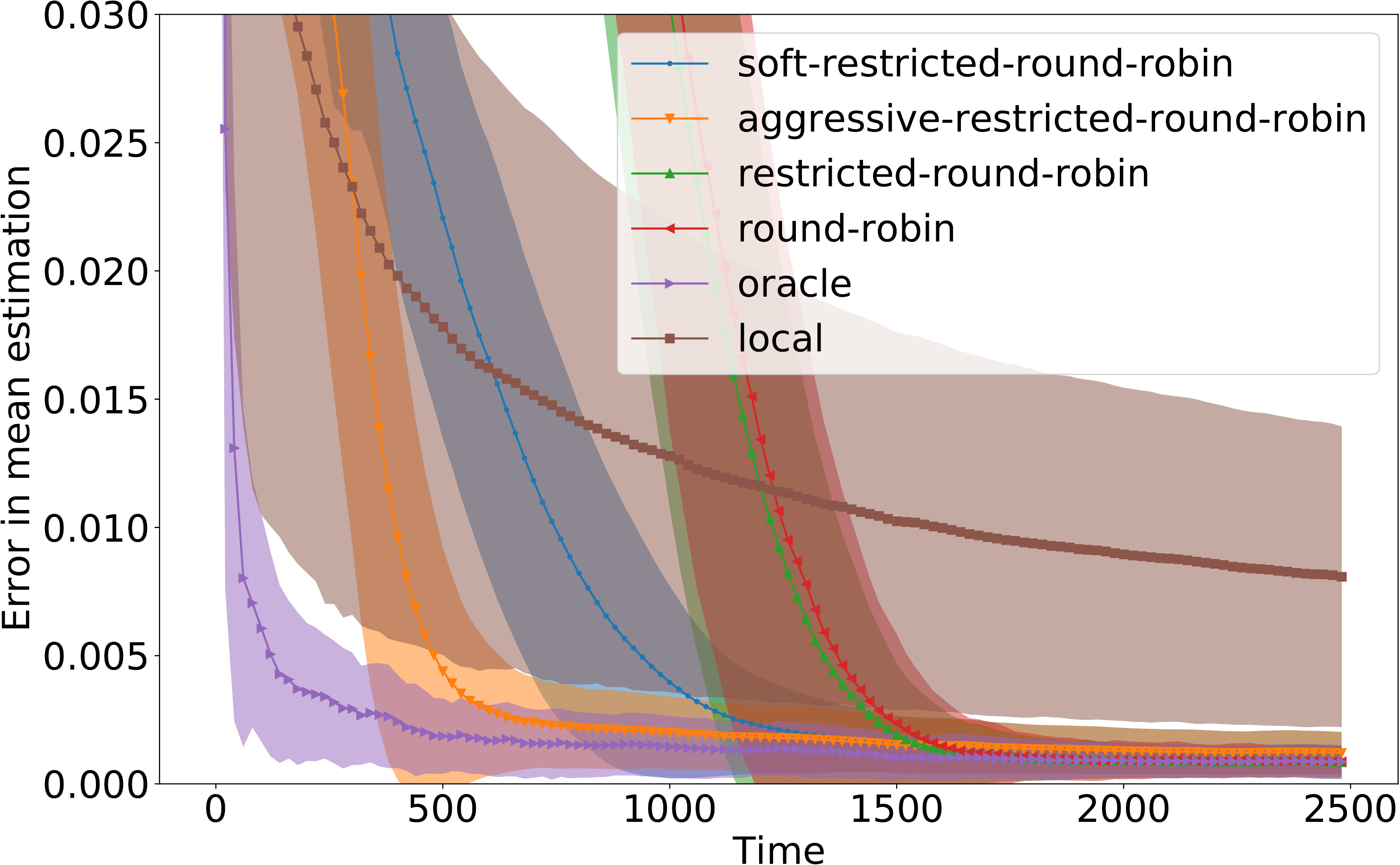}
    \label{fig:estimation3}}
    % \subfigure[Error in mean estimation for agents in class $0.8$]
    % {\includegraphics
    % [width=.33\textwidth]{Results/Class2-view-error_gaussian3.pdf}
    % \label{fig:estimation3_class2}}    
	\caption{Results across time on the 3-class problem (Gaussian
   distributions
   with true means $0.2$, $0.4$, $0.8$). Thanks to our collaborative
   algorithms (\SRRR,
\ARRR, \RRR, \RR), agents are able to estimate their true class
(Fig.~\ref{fig:precision3}) and thereby obtain accurate mean estimates
much more quickly than using purely local estimation (Fig.~\ref{fig:estimation3}).}
\end{figure*}

\begin{table}[t]
    \centering
    \small
    \rev{
    \begin{tabular}{llllllll}
        \toprule
        Algorithm & \multicolumn{2}{c}{class 0.2} & \multicolumn{2}{c}{class 0.4} & \multicolumn{2}{c}{class 0.8}\\
            & avg/std & max & avg/std & max & avg/std & max\\
        \midrule
        Round-Robin (RR) & $1378 \pm 194$& $2150$&$1382 \pm
        191$&$2159$ & $407 \pm 40$& $662$ \\
        Restricted-RR & $1376 \pm 211$& $2163$ & $1379 \pm 210$&
        $2245$ & $373 \pm 55$ & $934$\\
        %Soft-Restricted-RR & $1375.64 \pm 211.10$ & $1379.27 \pm 210.37$ & $373.48 \pm 55.15$\\
        %Aggressive-Restricted-RR& $1375.64 \pm 211.10$& $1379.27 \pm 210.37$ & $373.48 \pm 55.15$\\
        \midrule
        \midrule
        Theoretical Restricted-RR ($\zeta_a$) & \multicolumn{2}{c}{$3878$} & 
        \multicolumn{2}{c}{$3878$} & \multicolumn{2}{c}{$1085$} \\
        \bottomrule
    \end{tabular}
    }
    \caption{\rev{Empirical class estimation times of Round-Robin and
    Restricted-RR
    on the 3-class problem (Gaussian
   distributions with true means $0.2$, $0.4$, $0.8$). We report
   the average, standard deviation and maximum across agents and runs. We also report the high-probability class
   estimation time $\zeta_a$ given by our theory.}
    \label{tab:class_est}}
\end{table}

\rev{Table~\ref{tab:class_est} shows statistics about the average, standard
deviation and maximum time taken by an agent to correctly identify its
class. As expected, agents from class 0.8 identify their class much faster as
the gaps with respect to other classes are larger. Table~\ref{tab:class_est}
also reports the high-probability class estimation time $\zeta_a$ of \RRR
given by our theoretical analysis
(Theorem~\ref{theorem:class_estimation}). This theoretical value is rather
close to the maximum value we observe: although these values are not directly
comparable, it suggests that our analysis captures the
correct order of magnitude.}

\subsection{Mean Estimation}

We now turn to our main objective: mean estimation.
The error of an agent $a$ at time $t$ is evaluated as the absolute
difference
of its mean
estimate with its true mean:
\begin{align}
    \text{error}_{\oneagentindex}^{t}
    =
    |
    \rv{\mu}{\oneagentindex}{t}
    -
    \rv{\mu}{\oneagentindex}{}
    |.
\end{align}
Similar to above, we compute the average and standard deviation of this
quantity across runs,
and then for
each time step we report in
Figure~\ref{fig:estimation3} the average of these quantities across all
agents for the different algorithms (\SRRR, 
\ARRR, \RRR, \RR, \Oracle,  and \Local).

As expected, all variants of \algo suffer from mean estimation
bias in the early steps (due to optimistic class estimation). However, as
the estimated class of each agent gets more precise (see
Figure~\ref{fig:precision3}), agents progressively eliminate this bias  and
eventually learn estimates with
similar error and variance as the \Oracle
baseline. On the other hand,
\Local does not have estimation bias (hence achieves smaller error on average
in early rounds) but exhibits much higher variance, and its
average error converges very slowly towards zero. These results show the
ability of our collaborative algorithms to construct highly accurate
mean estimates much faster than without collaboration. We can also see that
$\SRRR$
and $
\ARRR$ converge much quicker to low error estimates than $\RRR$. This shows
that our proposed heuristic
weighting schemes successfully reduce the relative weight given to agents
that actually belong to different classes well before they are identified as
such with sufficient confidence. The aggressive weighting scheme is
observed to perform best in practice.

In the appendix, we plot the error in mean estimation for each class
separately and observe that agents with mean
$0.8$ (i.e., in the class that is easiest to discriminate from
others) are the fastest to reach highly accurate estimates, followed by those
with mean $0.4$, and finally those with mean $0.2$.

 \begin{table*}[t]
    \centering
    \small
    \begin{tabular}{lllll}
        \toprule
        Algorithm & \multicolumn{2}{c}{$\mathrm{conv}(0.1)$} &
        \multicolumn{2}{c}{$\mathrm{conv}(0.01)$}\\
        & avg/std & \rev{max} & avg/std & \rev{max}\\
        \midrule
        Round-Robin (RR) & $417 \pm 230$& \rev{$1239$} & $916 \pm 427$
        & \rev{$2088$}
        \\
        Restricted-RR & $405 \pm 227$& \rev{$1175$} & $894 \pm 438$ &
        \rev{$1925$}\\
        Soft-Restricted-RR & $82 \pm 48$ & \rev{$340$} & $608
        \pm 290$ & \rev{$1432$}\\
        Aggressive-Restricted-RR& $56 \pm 39$ & \rev{$340$}& $
        \mathbf{335 \pm 127}$ & \rev{$\mathbf{958}$}\\
        \midrule
        Local & $\mathbf{41 \pm 39}$ & \rev{$\mathbf{289}$}& $4494 \pm
        3945$ & \rev{$28616$} \\
        \midrule
        Oracle & $\mathit{5\pm 4}$& \rev{$\mathit{19}$}&$\mathit{98 \pm
        63}$ & \rev{$\mathit{303}$}\\
        \midrule
        \midrule
        \rev{Theoretical Restricted-RR ($\tau_a$)} & \multicolumn{2}{c}{
        \rev{$3878$}‌}& 
        \multicolumn{2}{c}{\rev{$33406$}} \\
        \rev{Theoretical Local ($\tau_a$)} & \multicolumn{2}{c}{\rev{$885$}}& 
        \multicolumn{2}{c}{\rev{$100216$}}\\
        \bottomrule
    \end{tabular}
    \caption{Empirical convergence times (see Eq.~\ref{eq:convergencetime3}) of
    different algorithms on the 3-class problem (Gaussian
   distributions with true means $0.2$, $0.4$, $0.8$) for a target estimation error of $\epsilon=0.1$ 
    (unfavorable regime) and
    $\epsilon=0.01$ (favorable regime). We report
   the average, standard deviation \rev{and maximum} across agents and runs.
   \rev{We also
   report the high-probability mean
   estimation times $\tau_a$ given by our theory for \RRR and \Local.} In line
   with
   our theory, we see that our
   approach largely
    outperforms the local estimation baseline in the favorable regime and
    remains competitive in the unfavorable regime.}
    \label{table:epsilonthreshold}
\end{table*}

Finally, we quantitatively compare the convergence time of different
algorithms with an empirical
measure inspired by our theoretical PAC criterion 
(Definition~\ref{def:pac}).
We define the 
\emph{empirical convergence} time of an agent as the earliest time
step where the
estimation
error of the agent always stays lower than some $\epsilon$:
\begin{gather}
% \mathrm{stab}_{\oneagentindex}(\epsilon)
% = \min
% \{ \tau\in\mathbbm{N} : 
% \forall t \geq \tau, |\rv{\mu}{\oneagentindex}{t} - \rv{\mu}{\oneagentindex}{t + 1}|  \leq 2\epsilon
% \},
% \label{eq:convergencetime2}
% \\
\mathrm{conv}_{a}(\epsilon)
= \min
\{
\tau \in \mathbbm{N}: \forall t \geq \tau,
\text{error}_{\oneagentindex}^{t}
\leq
\epsilon
\}.
\label{eq:convergencetime3}
\end{gather}
We denote by % $\mathrm{stab}(\epsilon)$ and
$\mathrm{conv}(\epsilon)$ the
average of the above quantity across all runs and all agents.

Table~\ref{table:epsilonthreshold} reports the average, standard
deviation and maximum empirical convergence time across agents and runs for
two values of $\epsilon$ \rev{(we also provide per-class tables in the
appendix)}.
These values
were chosen to reflect
the two different regimes suggested by our theoretical analysis. Indeed,
recall that our
theory gives a criterion to predict whether our collaborative algorithms will
outperform the \Local baseline: this is the case when the desired accuracy of
the solution $\epsilon$ is small enough for the given problem instance (see
Eq.~\ref{eq:epsilon_threshold}). For the problem considered here,
Eq.~\eqref{eq:epsilon_threshold} gives that \RRR will outperform \Local
for all agents as soon as $\epsilon$ is smaller than $0.049$.
We thus choose $\epsilon=0.01$
as the favorable regime (where we should beat \Local) and $\epsilon=0.1$ as
the unfavorable regime. \rev{We provide the corresponding high-probability
mean
estimation times $\tau_a$ of \RRR and \Local given by our theoretical analysis
(Theorem~\ref{thm:mean_estimation}).}

The results in Table~\ref{table:epsilonthreshold} are
consistent with our theory. All variants of our algorithms largely outperform
\Local
for $\epsilon=0.01$,\footnote{\rev{We had to run \Local for a much larger time
horizon of 30000 steps for all runs to converge to
accuracy $\epsilon=0.01$.}} while \Local is better for $\epsilon=0.1$ as
agents can
reach this precision using only their own samples faster than they can
reliably estimate
their class. Overall, $\RRR$
performs marginally better than \RR, while \SRRR and \ARRR significantly
outperform \RR and \RRR in both cases. \rev{Note that \SRRR and \ARRR perform
roughly the same as \Local in the unfavorable
regime, and get very close to the performance of the \Oracle baseline in the
favorable regime.} These results again show the relevance of our collaborative
algorithms and heuristic weighting schemes. 
% In case of the \Local, the algorithm requires more samples for $\epsilon=0.01$.
\rev{We observe that our theoretical results get looser as
$\epsilon\rightarrow0$, which is somewhat expected.}

%%% Local Variables:
%%% mode: latex
%%% TeX-master: "tmlr"
%%% ispell-local-dictionary: "english"
%%% End:

% !TEX root = tmlr.tex
\section{Extension to Imperfect Classes}
\label{sec:imperfect}

So far we have assumed that two agents are in the same class if their personal
distributions have \emph{exactly} the same mean, which can be restrictive for
some use-cases. In this section, we show that we can extend the problem setup
and our approach to the case where \emph{two agents are considered to be in
the same class if their means are close enough} and agents seek to 
\emph{estimate the mean of their class}.

Formally, we define a new notion of similarity class parameterized by a radius
$\eta$, which generalizes our previous notion introduced in
Definition~\ref{def:true_similarity}.
\begin{definition}
	Given $\eta>0$, the \emph{$\eta$-similarity class} of agent
	$\oneagentindex$ is given
	by:
	\begin{align*}
	\class{\eta,\oneagentindex} 
	= 
	\{
	\onememoryindex \in [\nagent]: 
	\Delta_{\oneagentindex, \onememoryindex} \leq \eta
	\},
	\end{align*}
	where $\Delta_{\oneagentindex,\onememoryindex}
	= 
	|
	\mu_\oneagentindex
	-
	\mu_\onememoryindex
	|$ is the gap between the means of agent $a$ and agent $l$.
	\label{def:eta_similarity}
\end{definition}

This notion of ``imperfect'' similarity class allows to capture situations
where \emph{clusters} of agents have similar (but not necessarily equal)
means. Such discrepancies between the means of agents in the same class may
for instance stem from the presence of local measurement bias
\citep[e.g., due to local variations in the environment, see][]
{bias_measure}. They can also be used to model groups of
agents with similar preferences, behavior, or goals, in applications like 
collaborative filtering \citep{DBLP:journals/advai/SuK09}, 

In this context, it is natural to slightly redefine the estimation objective. Instead of estimating its personal mean $\mu_a$ as considered so far,
each agent $a$ aims to estimate the mean of its class:
\begin{equation}
\label{eq:mean_class}
\rv{\mu}{\eta, \oneagentindex}{}=\frac{1}{|\class{\eta,\oneagentindex}|}\sum_
{l\in \class{\eta,\oneagentindex}} \mu_l.
\end{equation}
For instance, in the presence of (centered) local measurement bias, estimating
the class mean (instead of the local mean) allows to debias the estimate.

\begin{remark}[Non-separated clusters]
We do not formally require that the $\eta$-similarity classes form separated
clusters, in the sense that for
three distinct agents $a,l,i\in[A]$ we may have
simultaneously $i\in\class{\eta,a}$, $i\in \class{\eta,l}$ and $
\class{\eta,a}\neq \class{\eta,l}$. This happens when $\Delta_
{\oneagentindex,i}\leq \eta$, $\Delta_{l,i}\leq \eta$ and $\eta < \Delta_
{\oneagentindex,l}\leq 2\eta$. In this case, the ``class'' of an agent
simply corresponds to a ball of radius $\eta$ around its mean, which
potentially overlaps with others and thus violates the
transitivity property of equivalence classes. For consistency with the rest of
the paper and with an slight abuse of terminology, we continue to use the term
``class''. Although the case of separated clusters appears more natural, we
note that our proposed approach will still work in the non-separated setting,
in the sense that agents will correctly estimate the mean of their class as
defined in Eq.~\ref{eq:mean_class}.
\end{remark}

Based on the above, we can adapt the notion of optimistic similarity class 
(Definition~\ref{def:optimistic_similarity}) and the
condition on the number of samples required for this optimistic
class to coincide with the true class
(Definition~\ref{def:needed_samples}) by incorporating $\eta$.

\begin{definition} 
	\label{def:eta_class_heuristic}
	The \emph{$\eta$-optimistic similarity class}  from the perspective of agent
	$\oneagentindex$ at time $t$ is defined as:
	\begin{align*}
	\classt{\eta, \oneagentindex}{t}
	=
	\{
	\onememoryindex \in [\nagent]
	:
	\distance{\oneagentindex, \delta}{t}{\onememoryindex} \leq \eta
	% \land
	% \distance{\oneagentindex, +}{t}{\oneagentindex}{\onememoryindex} \leq \similaritythreshold
	\}.
	\end{align*}
\end{definition}

\begin{definition}
	From the perspective of agent $\oneagentindex$ and at time $t$, event $G_{\eta, \oneagentindex}^t$ is defined as:
	\begin{align}
	G_{\eta, \oneagentindex}^t
	=
	\underset{\onememoryindex \in [\nagent]}{\bigcap}
	\nobservation{\oneagentindex, \onememoryindex}{t}
	>
	\nobservation{\oneagentindex, \onememoryindex}{\eta},
	\end{align}
	\begin{equation*}
	\text{where ~~~~~~~} \nobservation{\oneagentindex, \onememoryindex}
	{\eta}
	= 
	\begin{cases}
	\lceil \CBinv{\delta}{\frac{\Delta_{\oneagentindex, \onememoryindex} - \eta}{4}}\rceil
	& \text { if }\onememoryindex \notin \class{\oneagentindex},
	\\
	\lceil \CBinv{\delta}{\frac{\Delta_{\eta, \oneagentindex} - \eta }{4}}\rceil
	& \text{ otherwise},
	\end{cases}
	\end{equation*}
	with  $\Delta_{\eta, \oneagentindex} = \min_{\onememoryindex \in 
		[\nagent]\backslash\class{\eta, \oneagentindex}} \Delta_{\oneagentindex, \onememoryindex}$.
	\label{def:eta_needed_samples}
\end{definition}

\begin{restatable}[Class membership rule]{lemma}{lemmathree}
	Under $E_{\oneagentindex}^{t} \land G_{\eta, \oneagentindex}^t$ and
	$\forall \onememoryindex \in [\nagent]$ and at time $t$:
	$\distance{\oneagentindex,\delta}{t}{\onememoryindex} > \eta \iff
	\onememoryindex \in [\nagent]\backslash\class{\eta, \oneagentindex}$.
	\label{lemma:eta_class-membership-rule}
\end{restatable}

We can see from the above that ruling out an agent $l$ from the optimistic
class $\class{\eta, \oneagentindex}$ requires more samples for larger $\eta$,
which is expected as the size of the confidence interval needs to be smaller
to make this decision reliably.

With these tools in place, we can use our collaborative mean
estimation algorithm ColME (Algorithm~\ref{algorithm:main}) presented before, 
with only minor modifications: we simply need to replace the
notion of optimistic similarity class by the $\eta$-version of
Definition~\ref{def:eta_class_heuristic}, and compute the estimate $\rv{\mu}
{\eta, \oneagentindex}{t}$ at time $t$ using a simple
\emph{class-uniform
weighting scheme}
% \begin{equation*}
	$\weighting{\oneagentindex, \onememoryindex}{t}
	=
	\frac{1}{|\classt{\eta, \oneagentindex}{t}|}$ to match the objective in
	Eq.~\ref{eq:mean_class}. We refer to this algorithm as \etaalgo. Note that
	$\eta$ becomes a parameter of the
	algorithm, allowing to choose the desired radius for the class structure.

We can now state the class and mean estimation complexity of \etaalgo. The
proofs can be found in the appendix.

\begin{restatable}[\etaalgo class estimation time complexity]{theorem}{thmthree}
	\label{theorem:eta_class_estimation}
	For any $\delta \in (0,1)$, employing \RRR query strategy, we have:
	\begin{equation}
	\pr{
		\exists t >\switchingtime{\oneagentindex}^{\eta}: \classt{\eta, \oneagentindex}{t}
		\neq \class{\eta, \oneagentindex}
	}
	\leq
	\frac{\delta}{8},
	% the probability of error is bounded by  $\frac{\delta}{4}$ where the class estimation time complexity for agent $\oneagentindex$ is defined as:
	\quad\text{with ~}
	\switchingtime{\oneagentindex}^{\eta}
	=
	\nobservation{\oneagentindex, \oneagentindex}{\eta} 
	+
	\nagent -1 
	-
	\sum_{\onememoryindex \in [\nagent]\setminus\class{\eta, \oneagentindex}}
	\indicator
	{\nobservation{\oneagentindex, \oneagentindex}{\eta}
		>
		\nobservation{\oneagentindex, \onememoryindex}{\eta} 
		+
		\nagent - 1 
	}.
	\label{eq:eta_switchingtime}
	\end{equation}
\end{restatable}

\begin{restatable}[\etaalgo mean estimation time complexity]{theorem}{thmfour}
	\label{thm:eta_mean_estimation}
	Given the risk parameter $\delta$, using the \RRR query strategy and
	class-uniform
	weighting (while employing $\class{\eta, \oneagentindex}$), the mean estimator $\rv{\mu}{\oneagentindex}{t}$ of agent
	$\oneagentindex$ is $(\epsilon,\frac{\delta}{4})$-convergent, that is:
	% converges to being $\epsilon-$close of the true mean $\rv{\mu}{\oneagentindex}{}$ with high probability $1 - \frac{\delta}{4}$:
	\begin{align}
	\mathbbm{P}
	\big(
	\forall t > \stoppingtime{\oneagentindex}^{\eta}
	:
	|\rv{\mu}{\eta, \oneagentindex}{t}
	-
	\rv{\mu}{\eta, \oneagentindex}{}|
	\leq
	\epsilon
	\big)
	>
	1 - \frac{\delta}{4},
	\quad \text{with ~} 
	% 	\rv{\mu}{\eta, \oneagentindex}{t}
	% 	=
	% 	\sum_{\onememoryindex \in \classt{\eta, \oneagentindex}{t}}
	% 	\frac{\rv{\bar{x}}{\oneagentindex, \onememoryindex}{t}}{|\classt{\eta, \oneagentindex}{t}|}
	% \text{ and } 
	\stoppingtime{\oneagentindex}^{\eta}
	=
	\max
	(
	\switchingtime{\oneagentindex}^{\eta}
	,
	\CBinv{\delta}{\epsilon} + |\class{\eta, \oneagentindex}| - 1
	).
	\label{eq:eta_tau_a}
	\end{align}
\end{restatable}

\rev{The class estimation result
(Theorem~\ref{theorem:eta_class_estimation})
is similar to its ``perfect'' class counterpart
(Theorem~\ref{theorem:class_estimation}) except
that it involves $\eta$-dependent quantities. The mean estimation result
(Theorem~\ref{thm:eta_mean_estimation}) differs slightly more from the perfect
class case (Theorem~\ref{thm:mean_estimation}) because the estimation
objective (see Eq.~\ref{eq:mean_class}) and weighting scheme are different.
But most importantly, we see that} for
large enough gaps
or small enough precision $\epsilon$ (similar to
Eq.~\ref{eq:epsilon_threshold}), we again achieve an optimal speed
since the time complexity of an oracle weighting baseline that would know the
true
classes beforehand is $\CBinv{\delta}{\epsilon} + |\class{\eta, \oneagentindex}| - 1$.

%%% Local Variables:
%%% mode: latex
%%% TeX-master: "supp"
%%% ispell-local-dictionary: "english"
%%% End:

% !TEX root = tmlr.tex

\section{Conclusion}
\label{sec:conclusion}
We have presented collaborative online algorithms where each agent learns the
set (class) of agents who
shares the same (or similar) objective and uses this knowledge to speed up the
estimation of its personalized mean. We
have provided PAC-style guarantees for the class and mean estimation time
complexity of our algorithms. In
addition, we have introduced a number of sample weighting mechanisms to
decrease the
bias of the estimates in the early rounds of learning, whose benefits are
illustrated empirically.

Our work initiates the study of online, collaborative and personalized
estimation and learning problems, which we believe to be a promising area for
future work. We outline a few interesting directions below.

\paragraph{\rev{Optimistic vs conservative.}} Instead of the optimistic
approach taken in this work, one could try to design a more
conservative
algorithm where an agent incorporates the estimate of another agent only when
it knows (with sufficient probability) that it belongs to the same class.
This would avoid introducing bias in the estimation in early steps. However, a
downside of such an approach is that agents would typically need some
knowledge of the gaps between their true means in order to determine when
another agent can be considered to be in the same class, which would be a big
practical limitation.

\rev{\paragraph{Large-scale variants.} While a simple way to scale
up our approach to a large number of agents is to have each
agent focus on a restricted subset of other agents (see
Remark~\ref{rem:scale}), an interesting direction to allow more exploration in
large-scale systems could rely on the idea of peer sampling \citep{jelasity2007gossip}, i.e., randomly sampling a few agents from time to time to discover potential new members of the class beyond the initial subset.
}

\paragraph{\rev{Handling data drift.}}
We would like to extend our approach to handle data drift, where the
means of
agents can change over time. Here, one could try to adapt ideas from
non-stationary bandits,  such as sliding-window UCB 
\citep{DBLP:conf/alt/GarivierM11} or UCB strategies mixed
with change-point detection algorithms \citep{pmlr-v89-cao19a}.

\paragraph{\rev{Privacy guarantees.}} \rev{In use cases where data is
considered sensitive (e.g., personal data), it is important to provide
strong privacy guarantees to the agents. While our approach only requires
agents to share aggregate quantities (see also
Remark~\ref{rem:sample_vs_com}), these may still leak sensitive information.
In future work, we would like to use tools from differential privacy
\citep{dwork2013Algorithmic}, such as the tree aggregation technique
for sharing cumulative sums in an online way \citep{treeagg2,treeagg1}, to
provide formal privacy guarantees and analyze the resulting trade-offs between
privacy and utility.}

\paragraph{\rev{Extensions to personalized learning tasks.}}
Finally, the problem could be extended to cases where
each agent aims to solve a personalized machine learning task \citep{vanhaesebrouck2017decentralized} based on
the data it receives online.
In that case, a structure in the distribution  conditioned by the outputs of
the learned models can be inferred, introducing an interesting
exploration-exploitation dilemma in the learning task.

%%% Local Variables:
%%% mode: latex
%%% TeX-master: "tmlr"
%%% ispell-local-dictionary: "english"
%%% End:

% !TEX root = tmlr.tex

\subsubsection*{Acknowledgments}
The authors thank the reviewers for their insightful comments that allowed to
improve the paper. This work was funded in part by Métropole Européenne de Lille (MEL), Inria, Université de Lille, and the I-SITE ULNE through the AI chair Apprenf number R-PILOTE-19-004-APPRENF.

\bibliography{references}
\bibliographystyle{tmlr}

\appendix
% !TEX root = tmlr.tex

\clearpage
\onecolumn

\appendix
\renewcommand{\thesection}{Appendix~\Alph{section}}
\renewcommand{\thesubsection}{\Alph{section}.\arabic{subsection}}

% \section*{SUPPLEMENTARY MATERIAL}
\section{Proof of Lemma~\ref{lemma:E_a_true}}

\begin{lemma}
	\label{lemma:timeuniformextension}
	Let $\rv{\mu}{\oneagentindex}{t}$ be the mean value of $t$ independent real-valued random variables with the true mean $\rv{\mu}{\oneagentindex}{}$ and is $\sigma$-sub-gaussian. For all $\delta \in (0, 1)$, it holds:
	\begin{align}
		\pr
		{
			\exists t \in \mathbbm{N},
			\rv{\mu}{\oneagentindex}{t}
			-
			\rv{\mu}{\oneagentindex}{}
			\geq
			\sigma
			\sqrt
			{
				\frac{2}{t}
				(
				1 + \frac{1}{t}
				)
				\ln(\sqrt{t+1}/\delta)
			}
		}
	\leq
	\delta\enspace ,
	\\
	\pr
	{
		\exists t \in \mathbbm{N},
		\rv{\mu}{\oneagentindex}{}
		-
		\rv{\mu}{\oneagentindex}{t}
		\geq
		\sigma
		\sqrt
		{
			\frac{2}{t}
			(
			1 + \frac{1}{t}
			)
			\ln(\sqrt{t+1}/\delta)
		}
	}
	\leq
	\delta \enspace .
	\end{align}
\end{lemma}
\begin{proof}
	The two inequalities are proved in the same way as a direct consequence
	of~\cite{maillard2019mathematics}, Lemma 2.7 therein. Let $Y_1, \dots,
	Y_t$ be a sequence of independent real-valued random variables where for each $s \leq t$, $Y_s$ has mean $\rv{\mu}{s}{}$ and is $\sigma_s$-sub-Gaussian,  then for all $\delta \in (0, 1)$, it holds that
	\begin{align*}
		\pr
		{
			\exists t \in \mathbbm{N},
			\sum_{s=1}^{t}
			(
			Y_s - \rv{\mu}{s}{}
			)
			\geq
			\sqrt
			{
				2\sum_{s=1}^{t}
				\sigma_s^2
				(
				1 + \frac{1}{t}
				)
				\ln(\sqrt{t+1}/\delta)
			}
		}
	\leq
	\delta\enspace .
	\end{align*}
	
	When all random variables $Y_s$ have the same mean $\mu_a$ and variance $\sigma$, we have
	\begin{align*}
	\pr
	{
		\exists t \in \mathbbm{N},
		\sum_{s=1}^{t}
		(
		Y_s - \rv{\mu}{a}{}
		)
		\geq
		\sqrt
		{
			2t\sigma^2
			(
			1 + \frac{1}{t}
			)
			\ln(\sqrt{t+1}/\delta)
		}
	}
	\leq
	\delta,
	\end{align*}
	
	Taking the average rather than the sum,	i.e. dividing both sides by $t$  we obtain that:
	\begin{align*}
	\pr
	{
		\exists t \in \mathbbm{N},
		\sum_{s=1}^{t}
		(
		\frac{Y_s}{t} - \frac{\rv{\mu}{a}{}}{t}
		)
		\geq
		\sqrt
		{
			\frac{2}{t}
			\sigma^2
			(
			1 + \frac{1}{t}
			)
			\ln(\sqrt{t+1}/\delta)
		}
	}
	\leq
	\delta,
	\end{align*}
	\begin{align*}
		\pr
		{
			\exists t \in \mathbbm{N},
			\sum_{s=1}^{t}
			\frac{Y_s}{t} 
			-
			\rv{\mu}{\oneagentindex}{}
			\geq
			\sqrt
			{
				\frac{2}{t}
				\sigma^2
				(
				1 + \frac{1}{t}
				)
				\ln(\sqrt{t+1}/\delta)
			}
		}
		\leq
		\delta\enspace .
	\end{align*}
	And denoting $\rv{\mu}{\oneagentindex}{t} = \sum_{s=1}^{t} 
	\frac{Y_s}{t} $, we conclude
	\begin{align*}
	\pr
	{
		\exists t \in \mathbbm{N},
		\rv{\mu}{a}{t}
		-
		\rv{\mu}{\oneagentindex}{}
		\geq
		\sigma
		\sqrt
		{
			\frac{2}{t}
			(
			1 + \frac{1}{t}
			)
			\ln(\sqrt{t+1}/\delta)
		}
	}
	\leq
	\delta\enspace .
	\end{align*}
\end{proof}

\lemmaone*
\begin{proof}
  Let us recall that $E_{\oneagentindex}=	\underset{ t \in \mathbbm{N}}{\bigcap}\underset{ \onememoryindex \in [\nagent]}{\bigcap}
		|
		\rv{\bar{x}}{\oneagentindex, \onememoryindex}{t}
		-
		\rv{\mu}{\onememoryindex}{}
		|
		\leq
		\CB{\delta}{\nobservation{\oneagentindex, \onememoryindex}{t}}$. 
	Then
	\begin{align*}
	\pr
	{
		E_{\oneagentindex}
	}
	&=
	1 
	-
	\pr
	{
		\bar{E_{\oneagentindex}}
	},\\
	&=
	1 
	-
	\pr
	{
		\exists t \in \mathbbm{N},\, 
		\exists \onememoryindex \in [\nagent] :
%		\exists \nobservation{\oneagentindex, \onememoryindex}{t} \leq m:
		|
		\rv{\bar{x}}{\oneagentindex, \onememoryindex}{t}
		-
		\rv{\mu}{\onememoryindex}{}
		|
		>
		\CB{\delta}{\nobservation{\oneagentindex, \onememoryindex}{t}}
	},\\
	&\geq
	1
	-
	\sum_{\onememoryindex \in [\nagent]}
	\pr
	{
		\exists t \in \mathbbm{N}:
		|
		\rv{\bar{x}}{\oneagentindex, \onememoryindex}{t}
		-
		\rv{\mu}{\onememoryindex}{}
		|
		>
		\CB{\delta}{\nobservation{\oneagentindex, \onememoryindex}{t}}
	}.
	\end{align*}
	% without loss of generality $\rv{\bar{x}}{\oneagentindex, \onememoryindex}{t} = \rv{\hat{x}}{m}{}$ and
        defining $\gamma(\delta) = \frac{\delta}{8\times \nagent}$ and using Lemma~\ref{lemma:timeuniformextension} 
	\begin{align*}
	\pr{E_{\oneagentindex}}
	&\geq
	1 
	-
	\sum_{\onememoryindex \in [\nagent]}
	\mathbbm{P}
	\Big(
	\exists t \in \mathbbm{N}:
	|
	\rv{\bar{x}}{\oneagentindex, \onememoryindex}{t}
	-
	\rv{\mu}{\onememoryindex}{}
	|
	>
	\sigma
	\sqrt
	{
		\frac{2}{\nobservation{\oneagentindex, \onememoryindex}{t}}
		\times
		(
		1
		+ \frac{1}{\nobservation{\oneagentindex, \onememoryindex}{t}}
		)
		\ln(
		\sqrt
		{
			\nobservation{\oneagentindex, \onememoryindex}{t} + 1
		}
		/
		\gamma(\delta)
		)
	}
	\Big),
	\\
	&\geq
	1 
	-
	\sum_{\onememoryindex \in [\nagent]}
	\gamma(\delta)
	=
	1 
	-
	\sum_{\onememoryindex \in [\nagent]}
	\frac{\delta}{8 \nagent}
	=
	1
	-
	\frac{\delta}{8}.\qedhere
	\end{align*}
\end{proof}

\section{Proof of Theorem~\ref{theorem:class_estimation}}
In this section, we prove Theorem~\ref{theorem:class_estimation}. We first show Lemma~\ref{lemma:class-membership-rule} using two technical lemmas. We then prove Lemma~\ref{lemma:gat_nalt_nalstar}, which combined with Lemma~\ref{lemma:class-membership-rule}, yields the main result (Theorem~\ref{theorem:class_estimation}). Let us first remark that

\begin{align*}
\distance{\oneagentindex, \delta}{t}{\onememoryindex}
&=
|
\rv{\bar{x}}{\oneagentindex, \oneagentindex}{t}
-
\rv{\bar{x}}{\oneagentindex, \onememoryindex}{t}
|
-
\CB{\delta}{\nobservation{\oneagentindex, \oneagentindex}{t}}
-
\CB{\delta}{\nobservation{\oneagentindex, \onememoryindex}{t}},\\
&=
|
(\rv{\bar{x}}{\oneagentindex, \oneagentindex}{t} - \rv{\mu}{\oneagentindex}{})
-
(\rv{\bar{x}}{\oneagentindex, \onememoryindex}{t} -\rv{\mu}{\onememoryindex}{})
+(\rv{\mu}{\oneagentindex}{} -\rv{\mu}{\onememoryindex}{} )
|
-
\CB{\delta}{\nobservation{\oneagentindex, \oneagentindex}{t}}
-
\CB{\delta}{\nobservation{\oneagentindex, \onememoryindex}{t}}.
\end{align*}

As a consequence we have

\begin{align}
\distance{\oneagentindex, \delta}{t}{\onememoryindex}
&\leq
\Delta_{\oneagentindex, \onememoryindex}
+
|
\rv{\bar{x}}{\oneagentindex, \oneagentindex}{t}
-
\rv{\mu}{\oneagentindex}{}
|
+
|
\rv{\bar{x}}{\oneagentindex, \onememoryindex}{t}
-
\rv{\mu}{\onememoryindex}{}
|
-
\CB{\delta}{\nobservation{\oneagentindex, \oneagentindex}{t}}
-
\CB{\delta}{\nobservation{\oneagentindex, \onememoryindex}{t}} \label{eq:updelta}.\\
\distance{\oneagentindex, \delta}{t}{\onememoryindex}
&\geq
\Delta_{\oneagentindex, \onememoryindex}
-
|
\rv{\bar{x}}{\oneagentindex, \oneagentindex}{t}
-
\rv{\mu}{\oneagentindex}{}
|
-
|
\rv{\bar{x}}{\oneagentindex, \onememoryindex}{t}
-
\rv{\mu}{\onememoryindex}{}
|
-
\CB{\delta}{\nobservation{\oneagentindex, \oneagentindex}{t}}
-
\CB{\delta}{\nobservation{\oneagentindex, \onememoryindex}{t}} \label{eq:lowdelta}.
\end{align}

\begin{lemma}\label{lem:nstar}
	Under $E_{\oneagentindex}$,   $\forall l\in[\nagent]$,  if $l\not\in \class{\oneagentindex}$ then $\forall \nobservation{\oneagentindex, \onememoryindex}{t}\geq\nobservation{\oneagentindex, \onememoryindex}{\star} = \lceil\CBinv{\delta}{\frac{\Delta_{\oneagentindex, \onememoryindex}}{4}}\rceil$ we have  $\distance{\oneagentindex,\delta}{t}{\onememoryindex}>0$.
\end{lemma}

\begin{proof}
	Under $E_{\oneagentindex}$, we have $|\rv{\bar{x}}{\oneagentindex, \onememoryindex}{t}-\rv{\mu}{\onememoryindex}{}|\leq\CB{\delta}{\nobservation{\oneagentindex, \onememoryindex}{t}}$ and $|\rv{\bar{x}}{\oneagentindex, \oneagentindex}{t}-\rv{\mu}{\oneagentindex}{}|\leq\CB{\delta}{\nobservation{\oneagentindex, \oneagentindex}{t}}$. Since $\nobservation{\oneagentindex, \oneagentindex}{t} \geq \nobservation{\oneagentindex, \onememoryindex}{t}$, we also have $\CB{\delta}{\nobservation{\oneagentindex, \oneagentindex}{t}}\leq \CB{\delta}{\nobservation{\oneagentindex, \onememoryindex}{t}}$.  Hence, using~\eqref{eq:lowdelta}, $
	\distance{\oneagentindex, \delta}{t}{\onememoryindex}
	\geq
	\Delta_{\oneagentindex, \onememoryindex}
	-
	2
	\CB{\delta}{\nobservation{\oneagentindex, \oneagentindex}{t}}
	-
	2
	\CB{\delta}{\nobservation{\oneagentindex, \onememoryindex}{t}} \geq
	\Delta_{\oneagentindex, \onememoryindex}
	-
	4
	\CB{\delta}{\nobservation{\oneagentindex, \onememoryindex}{t}}$.  If $l\in \class{\oneagentindex}$ then $\Delta_{\oneagentindex, \onememoryindex}=0$ and since  $\CB{\delta}{\nobservation{\oneagentindex, \onememoryindex}{t}}> 0$ we cannot ensure that $\Delta_{\oneagentindex, \onememoryindex} - 4 \CB{\delta}{\nobservation{\oneagentindex, \onememoryindex}{t}}>0$. If $l\not\in \class{\oneagentindex}$ then
	to ensure that $\distance{\oneagentindex, \delta}{t}{\onememoryindex}\geq    \Delta_{\oneagentindex, \onememoryindex} - 4 \CB{\delta}{\nobservation{\oneagentindex, \onememoryindex}{t}}>0$, we need that $
	\frac{  \Delta_{\oneagentindex, \onememoryindex}}{4}
	>
	\CB{\delta}{\nobservation{\oneagentindex, \onememoryindex}{t}}$ and hence $ \nobservation{\oneagentindex, \onememoryindex}{\star} = \lceil\CBinv{\delta}{\frac{\Delta_{\oneagentindex, \onememoryindex}}{4}}\rceil$.\footnote{In  extremely rare cases, the expression $\CBinv{\delta}{\frac{\Delta_{\oneagentindex, \onememoryindex}}{4}}$ could be an integer and we should add 1 to get a strict inequality. But for conciseness of  the expression, we omit the +1 in the definition of $n_{a,l}^\star$.}
	% are not using the precise value of $\nobservation{\oneagentindex, \onememoryindex}{\star} = \lfloor\CBinv{\delta}{\frac{\Delta_{\oneagentindex, \onememoryindex}}{4}}\rfloor + 1$.}
\end{proof}

\begin{lemma}
	Under  $E_{\oneagentindex}$,  $\forall l\in[\nagent]$,  $\forall t\in \mathbb N$, if $l\in \class{\oneagentindex}$ then  $\distance{\oneagentindex, \delta}{t}{\onememoryindex} \leq 0$.\label{lem:in_ca}
\end{lemma}
\begin{proof}
	Again, recall that  under $E_{\oneagentindex}^t$,  we have $|\rv{\bar{x}}{\oneagentindex, \onememoryindex}{t}-\rv{\mu}{\onememoryindex}{}|\leq\CB{\delta}{\nobservation{\oneagentindex, \onememoryindex}{t}}$ and $|\rv{\bar{x}}{\oneagentindex, \oneagentindex}{t}-\rv{\mu}{\oneagentindex}{}|\leq\CB{\delta}{\nobservation{\oneagentindex, \oneagentindex}{t}}$. Hence, using~\eqref{eq:updelta}, $
	\distance{\oneagentindex, \delta}{t}{\onememoryindex}
	\leq
	\Delta_{\oneagentindex, \onememoryindex}
	+
	\CB{\delta}{\nobservation{\oneagentindex, \oneagentindex}{t}}
	+
	\CB{\delta}{\nobservation{\oneagentindex, \onememoryindex}{t}}
	-
	\CB{\delta}{\nobservation{\oneagentindex, \oneagentindex}{t}}
	-
	\CB{\delta}{\nobservation{\oneagentindex, \onememoryindex}{t}}
	= \Delta_{\oneagentindex, \onememoryindex}$. 
	If $l\in \class{\oneagentindex}$ then $\Delta_{\oneagentindex, \onememoryindex}=0$ and thus $\distance{\oneagentindex, \delta}{t}{\onememoryindex}\leq0$.
	
\end{proof}

We can now prove Lemma~\ref{lemma:class-membership-rule}, which we restate
here for convenience.

\lemmatwo*
\begin{proof}
	From Lemma~\ref{lem:in_ca}, we directly have one implication. For the other one,   if $l\not \in \class{\oneagentindex}$ because  $G_\oneagentindex^t$ holds,  we have $\forall \onememoryindex\in[\nagent]$, $\nobservation{\oneagentindex, \onememoryindex}{t}\geq\nobservation{\oneagentindex, \onememoryindex}{\star}$, therefore we can  apply Lemma~\ref{lem:nstar} and we directly have $\distance{\oneagentindex, \delta}{t}{\onememoryindex}>0$. 
\end{proof}

\begin{lemma}
	\label{lemma:gat_nalt_nalstar}
	Under $E_{\oneagentindex}$, and using \RRR algorithm, $G_{\oneagentindex}^t$ holds when $t > \switchingtime{\oneagentindex}$	where
	\begin{equation*}
	\switchingtime{\oneagentindex}
	=
	\nobservation{\oneagentindex, \oneagentindex}{\star}
	-1
	+
	\nagent
	-
	\sum_{\onememoryindex \in [\nagent]\setminus\class{\oneagentindex}}
	\indicator{\nobservation{\oneagentindex, \oneagentindex}{\star} > \nobservation{\oneagentindex, \onememoryindex}{\star}-1 +\nagent}.
	\end{equation*}
\end{lemma}
\begin{proof}
	According to Algorithm~\ref{algorithm:main}, $\classt{\oneagentindex}{0} = [\nagent]$ and an agent is eliminated from set $\classt{\oneagentindex}{t}$ at time $t$ if $\distance{\oneagentindex,\delta}{t}{\onememoryindex}=
	|
	\rv{\bar{x}}{\oneagentindex, \onememoryindex}{t}
	-
	\rv{\bar{x}}{\oneagentindex, \oneagentindex}{t}
	|
	-
	\CB{\delta}{\nobservation{\oneagentindex, \oneagentindex}{t}}
	-
	\CB{\delta}{\nobservation{\oneagentindex, \onememoryindex}{t}} >  0$.
	According to Lemma~\ref{lem:nstar}, the time required to eliminate agent $\onememoryindex$ from the class $\class{\oneagentindex}$ is at least $\nobservation{\oneagentindex, \onememoryindex}{\star}$. If agent $l$ is queried at time $\nobservation{\oneagentindex, \onememoryindex}{\star} - 1 $, then using \RRR (or round robin), we are sure that it will be removed from $\classt{\oneagentindex}{t}$ for all $t$ larger than $\nobservation{\oneagentindex, \onememoryindex}{\star} - 1 +\nagent$. 
	
	Let us consider $h$ being an agent such that $\nobservation{\oneagentindex, h}{\star} = \max_{\onememoryindex \in [\nagent]\backslash\class{\oneagentindex}}\nobservation{\oneagentindex, \onememoryindex}{\star}$. By definition,  $\Delta_{\oneagentindex, h} = \min_{\onememoryindex \in   [\nagent]\backslash\class{\oneagentindex}} \Delta_{\oneagentindex, \onememoryindex}$ and $\nobservation{\oneagentindex, h}{\star}$ can be denoted by $\nobservation{\oneagentindex, \oneagentindex}{\star}$. 
	
	In the case of round robin, we are sure that $G_\oneagentindex^t$ will be true when $t\geq \nobservation{\oneagentindex, \oneagentindex}{\star} -1 + \nagent$. But using \RRR, since the loop ignores agents not in $\classt{\oneagentindex}{t}$, we have that  $G_{\oneagentindex}^t$ holds when $t > \switchingtime{\oneagentindex}$ where
	
	\begin{equation*}
	\switchingtime{\oneagentindex}
	=
	\nobservation{\oneagentindex, \oneagentindex}{\star}
	-1
	+
	\nagent
	-
	\sum_{\onememoryindex \in [\nagent]\setminus\class{\oneagentindex}}
	\indicator{\nobservation{\oneagentindex, \oneagentindex}{\star} > 
	\nobservation{\oneagentindex, \onememoryindex}{\star}-1 +\nagent}.\qedhere
	\end{equation*}
\end{proof}

Finally, we use the above lemmas to prove
Theorem~\ref{theorem:class_estimation}, which we restate below for
convenience.

\thmone*
\begin{proof}
	From Lemma~\ref{lemma:gat_nalt_nalstar} and Lemma~\ref{lemma:class-membership-rule}, we deduce that if  $E_\oneagentindex$ holds and knowing that $\classt{\oneagentindex}{t} =\{\onememoryindex \in [\nagent]: \distance{\oneagentindex, \delta}{t}{\onememoryindex} \leq 0\}$ then $\forall t > \switchingtime{\oneagentindex} $, $\class{\oneagentindex}=\classt{\oneagentindex}{t}$. Hence, $\pr{\forall t > \switchingtime{\oneagentindex},\,  \class{\oneagentindex}=\classt{\oneagentindex}{t}} \geq \pr{E_\oneagentindex} \geq 1 - \delta/8$ using Lemma~\ref{lemma:E_a_true}. 
\end{proof}

\section{Proof of Theorem~\ref{thm:mean_estimation}}
In this section we detail the proof of Theorem~\ref{thm:mean_estimation},
about the PAC mean estimation properties of the ColME strategy. We restate the
theorem below for convenience.

\thmtwo*
\begin{proof}
	Let us assume that at time $t$ we have $\class{\oneagentindex}^t = \class{\oneagentindex}$ .  Therefore
	\begin{equation*}
	\rv{\mu}{\oneagentindex}{t}
	=
	\sum_{\onememoryindex \in \class{\oneagentindex}}
	\rv{\bar{x}}{\oneagentindex, \onememoryindex}{t}
	\weighting{\oneagentindex, \onememoryindex}{t}
	=
	\frac{ \sum_{\onememoryindex \in \class{\oneagentindex}}
		\rv{\bar{x}}{\oneagentindex, \onememoryindex}{t}
		\nobservation{\oneagentindex, \onememoryindex}{t}}{\sum_{\onememoryindex \in \class{\oneagentindex}} \nobservation{\oneagentindex, \onememoryindex}{t}}.
	\end{equation*}
	Remark that $\sum_{\onememoryindex \in \class{\oneagentindex}}
	\rv{\bar{x}}{\oneagentindex, \onememoryindex}{t}
	\nobservation{\oneagentindex, \onememoryindex}{t}$ is the sum of all $\nobservation{\oneagentindex, \onememoryindex}{t}$ samples received by all agents $\onememoryindex$ in $\class{\oneagentindex}$. In  other words, $\rv{\mu}{\oneagentindex}{t}$ is the  estimation of $\mu_\oneagentindex$ with $\sum_{\onememoryindex \in \class{\oneagentindex}} \nobservation{\oneagentindex, \onememoryindex}{t}$ examples. Hence in order to have $|\rv{\mu}{\oneagentindex}{t} - \rv{\mu}{\oneagentindex}{}| \leq \epsilon$ when $E_{\oneagentindex}$ holds, we should have $\beta(\sum_{\onememoryindex \in \class{\oneagentindex}} \nobservation{\oneagentindex, \onememoryindex}{t})\leq \epsilon$. Let us see at what time denoted by $n_{\epsilon,a}$ we have $\lceil\beta^{-1}(\epsilon)\rceil=\sum_{\onememoryindex \in \class{\oneagentindex}} \nobservation{\oneagentindex, \onememoryindex}{t}$.  With Algorithm~\ref{algo:class_estimation} using \RRR, we know that when  $\class{\oneagentindex}^t = \class{\oneagentindex}$, then only members of $\class{\oneagentindex}$ are queried. Therefore,
	\begin{gather*}
	\lceil\beta^{-1}(\epsilon)\rceil = n_{\epsilon,a} + ( n_{\epsilon,a}-1)+\dots+( n_{\epsilon,a}-|\class{\oneagentindex}| + 1)=|\class{\oneagentindex}|n_{\epsilon,a}-\frac{|\class{\oneagentindex}|-1}{2}|\class{\oneagentindex}|,\\
	n_{\epsilon,a} = \frac{\lceil\beta^{-1}(\epsilon)\rceil}{|\class{\oneagentindex}|}+\frac{|\class{\oneagentindex}|-1}{2}.
	\end{gather*}
	
	As a summary,  if $E_\oneagentindex$ holds, then we have $\forall t\geq n_{\epsilon,\oneagentindex}$,   $\class{\oneagentindex}^t = \class{\oneagentindex}$ implies that $|\rv{\mu}{\oneagentindex}{t} - \rv{\mu}{\oneagentindex}{}| \leq \epsilon$. 
	Now, following Theorem~\ref{theorem:class_estimation}, we have $\pr{
		\exists t >\switchingtime{\oneagentindex}: \classt{\oneagentindex}{t}
		\neq \class{\oneagentindex}
	}
	\leq
	\frac{\delta}{8}$. Since $\stoppingtime{\oneagentindex}=\max( \switchingtime{\oneagentindex}, n_{\epsilon,\oneagentindex})\geq \switchingtime{\oneagentindex}$, then  $\pr{\exists t > \stoppingtime{\oneagentindex}: |\rv{\mu}{\oneagentindex}{t}-\rv{\mu}{\oneagentindex}{}| > \epsilon}\leq \frac{\delta}{8}+\pr{\bar E_\oneagentindex}=\frac{\delta}{4}$.
\end{proof}

\section{Proof of Theorem~\ref{theorem:eta_class_estimation}}
The
proof of Theorem~\ref{theorem:eta_class_estimation} follows the same step as
that of Theorem~\ref{theorem:class_estimation},
up to replacing the $0$ threshold by $\eta$. We only state the intermediate
lemmas (which are adaptations of
Lemmas~\ref{lem:nstar}-\ref{lem:in_ca}-\ref{lemma:gat_nalt_nalstar}) and omit
the detailed proof. 

\begin{lemma}\label{lem:eta_nstar}
  Under $E_{\oneagentindex}$,   $\forall l\in[\nagent]$,  if $l\not\in \class{\eta, \oneagentindex}$ then $\forall \nobservation{\oneagentindex, \onememoryindex}{t}\geq\nobservation{\oneagentindex, \onememoryindex}{\eta} = \lceil\CBinv{\delta}{\frac{\Delta_{\oneagentindex, \onememoryindex} - \eta}{4}}\rceil$ we have  $\distance{\oneagentindex,\delta}{t}{\onememoryindex}>\eta$.
\end{lemma}

\begin{lemma}
  Under  $E_{\oneagentindex}$,  $\forall l\in[\nagent]$,  $\forall t\in \mathbb N$, if $l\in \class{\eta, \oneagentindex}$ then  $\distance{\oneagentindex, \delta}{t}{\onememoryindex} \leq \eta$.\label{lem:eta_in_ca}
\end{lemma}

\begin{lemma}
  \label{lemma:eta_gat_nalt_nalstar}
  Under $E_{\oneagentindex}$, and using \RRR algorithm, $G_{\eta, \oneagentindex}^t$ holds when $t > \switchingtime{\oneagentindex}^{\eta}$	where
  \begin{equation*}
    \switchingtime{\oneagentindex}^{\eta}
    =
    \nobservation{\oneagentindex, \oneagentindex}{\eta}
    -1
    +
    \nagent
    -
    \sum_{\onememoryindex \in [\nagent]\setminus\class{\eta, \oneagentindex}}
    \indicator{\nobservation{\oneagentindex, \oneagentindex}{\eta} > \nobservation{\oneagentindex, \onememoryindex}{\eta}-1 +\nagent}.
  \end{equation*}
\end{lemma}

\section{Proof of Theorem~\ref{thm:eta_mean_estimation}}
In this section we detail the proof of Theorem~\ref{thm:eta_mean_estimation},
about the PAC mean estimation properties of the $\eta$-ColME strategy. We
restate the theorem below for convenience.

\thmfour*
\begin{proof}
  Since $t > \stoppingtime{\oneagentindex}^{\eta} > \switchingtime{\oneagentindex}^{\eta}$, at time $t$ we have $\class{\eta, \oneagentindex}^t = \class{\eta, \oneagentindex}$ .  Therefore
  \begin{equation*}
    \rv{\mu}{\oneagentindex}{t}
    =
    \sum_{\onememoryindex \in \class{\eta, \oneagentindex}}
    \rv{\bar{x}}{\oneagentindex, \onememoryindex}{t}
    \weighting{\oneagentindex, \onememoryindex}{t}
    =
    \frac{ \sum_{\onememoryindex \in \class{\eta ,\oneagentindex}}
      \rv{\bar{x}}{\oneagentindex, \onememoryindex}{t}}{|\class{\eta, \oneagentindex}|}.
  \end{equation*}
  Remark that $\rv{\mu}{\oneagentindex}{t}$ is not equivalent to the average
  of all the samples of agents in $\class{\eta, \oneagentindex}$: it is the
  average of
  the mean values for each agent in $\class{\eta, \oneagentindex}$.
  Therefore, although some agents may have more samples than the others, all
  are assigned uniform weights. We would like to have $|\rv{\mu}
  {\eta, \oneagentindex}{t} - \rv{\mu}{\eta, \oneagentindex}{}| \leq
  \epsilon$. When $E_{\oneagentindex}$ holds, we can rewrite this as
  \begin{equation*}
  	|\rv{\mu}{\eta, \oneagentindex}{t}
  	-
  	\rv{\mu}{\eta, \oneagentindex}{}|
  	=
  	|
  	\frac{1}{|\class{\eta, \oneagentindex}|}
  	\sum_{\onememoryindex \in \class{\eta, \oneagentindex}}
  	\rv{\bar{x}}{\oneagentindex, \onememoryindex}{t}
  	-
  	\rv{\mu}{\onememoryindex}{}
  	|
  	\leq
  	\frac{1}{|\class{\eta, \oneagentindex}|}
  	\sum_{\onememoryindex \in \class{\eta, \oneagentindex}}
  	|
  	\rv{\bar{x}}{\oneagentindex, \onememoryindex}{t}
  	-
  	\rv{\mu}{\onememoryindex}{}
  	|
  	\leq
  	\epsilon
  \end{equation*}
  Therefore, we need:
  \begin{equation*}
  	\sum_{\onememoryindex \in \class{\eta, \oneagentindex}}
  	|
  	\rv{\bar{x}}{\oneagentindex, \onememoryindex}{t}
  	-
  	\rv{\mu}{\onememoryindex}{}
  	|
  	\leq
  	|\class{\eta, \oneagentindex}| \times \epsilon,
  \end{equation*}
  A sufficient condition for the above inequality to hold is to ensure that
  each term is bounded by $\epsilon$:
  \begin{equation}
  	\forall \onememoryindex \in \class{\eta, \oneagentindex}:
  	|
  	\rv{\bar{x}}{\oneagentindex, \onememoryindex}{t}
  	-
  	\rv{\mu}{\onememoryindex}{}
  	|
  	\leq
  	\epsilon
  	\label{eq:eta_forall_centers_less_than_epsilon}
  \end{equation}
  This is achieved when $\CB
  {\delta}{\nobservation{\oneagentindex, \onememoryindex}{t}} < \epsilon$ for
  all $\onememoryindex \in \class{\eta, \oneagentindex}$. Since we are using
  \RRR and also that $\classt{\eta, \oneagentindex}{t} = \class{\eta,
  \oneagentindex}$, the number of samples required for each agent in $
  \class{\eta, \oneagentindex}$ are $n_{\oneagentindex, 1}^{t}$, $n_
  {\oneagentindex, 1}^{t} - 1$, $n_{\oneagentindex, 1}^{t} - 2$, $\dots$, $n_
  {\oneagentindex, 1}^{t} - |\class{\eta, \oneagentindex}| + 1$ where we
  consider the one with the maximum number of observations to have index 1 for
  notation simplicity (which corresponds to index $\oneagentindex$). For
  Eq.~\ref{eq:eta_forall_centers_less_than_epsilon} to hold, it is thus
  sufficient to have:
  \begin{equation*}
  	\CBinv{\delta}{\epsilon} < n_{\oneagentindex, \oneagentindex}^t - |\class{\eta, \oneagentindex}| + 1
  \end{equation*}
  \begin{equation*}
  	\CBinv{\delta}{\epsilon} + |\class{\eta, \oneagentindex}| - 1 < n_{\oneagentindex, \oneagentindex}^t 
  \end{equation*}
  Therefore $\stoppingtime{\oneagentindex}^{\eta} = \max(\switchingtime{\oneagentindex}^{\eta}, \CBinv{\delta}{\epsilon} + |\class{\eta, \oneagentindex}| - 1)$.

  As a summary,  if $E_\oneagentindex$ holds, then we have $\forall t\geq \stoppingtime{\oneagentindex}^{\eta}$,   $\class{\eta, \oneagentindex}^t = \class{\eta, \oneagentindex}$ implies that $|\rv{\mu}{\eta, \oneagentindex}{t} - \rv{\mu}{\eta, \oneagentindex}{}| \leq \epsilon$. 
  Now, following Theorem~\ref{theorem:eta_class_estimation}, we have $\pr{
      \exists t >\switchingtime{\oneagentindex}^{\eta}: \classt{\eta, \oneagentindex}{t}
      \neq \class{\eta, \oneagentindex}
    }
    \leq
    \frac{\delta}{8}$. Since $\stoppingtime{\oneagentindex}^{\eta}=\max( \switchingtime{\oneagentindex}^{\eta}, n_{\epsilon,\oneagentindex}^{\eta})\geq \switchingtime{\oneagentindex}^{\eta}$, then  $\pr{\exists t > \stoppingtime{\oneagentindex}^{\eta}: |\rv{\mu}{\oneagentindex}{t}-\rv{\mu}{\eta, \oneagentindex}{}| > \epsilon}\leq \frac{\delta}{8}+\pr{\bar E_\oneagentindex}=\frac{\delta}{4}$.
\end{proof}

\begin{figure}[t]
    \centering
    \subfigure[First class (mean $0.2$)]{\includegraphics[scale=0.20]
    {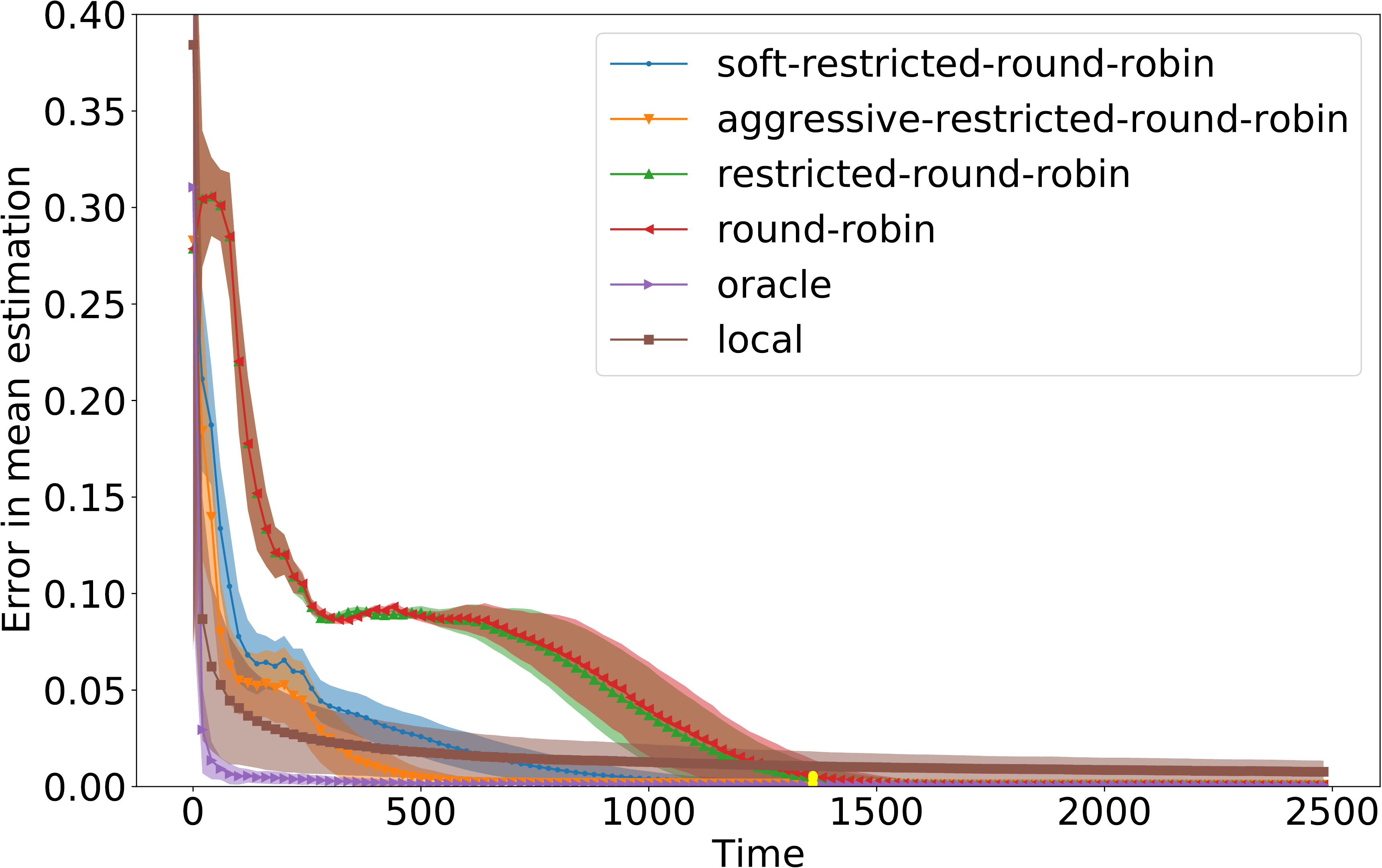}
        \label{fig:estimation_gaussian3_0}}
    \subfigure[Second class (mean $0.4$)]{\includegraphics[scale=0.20]
    {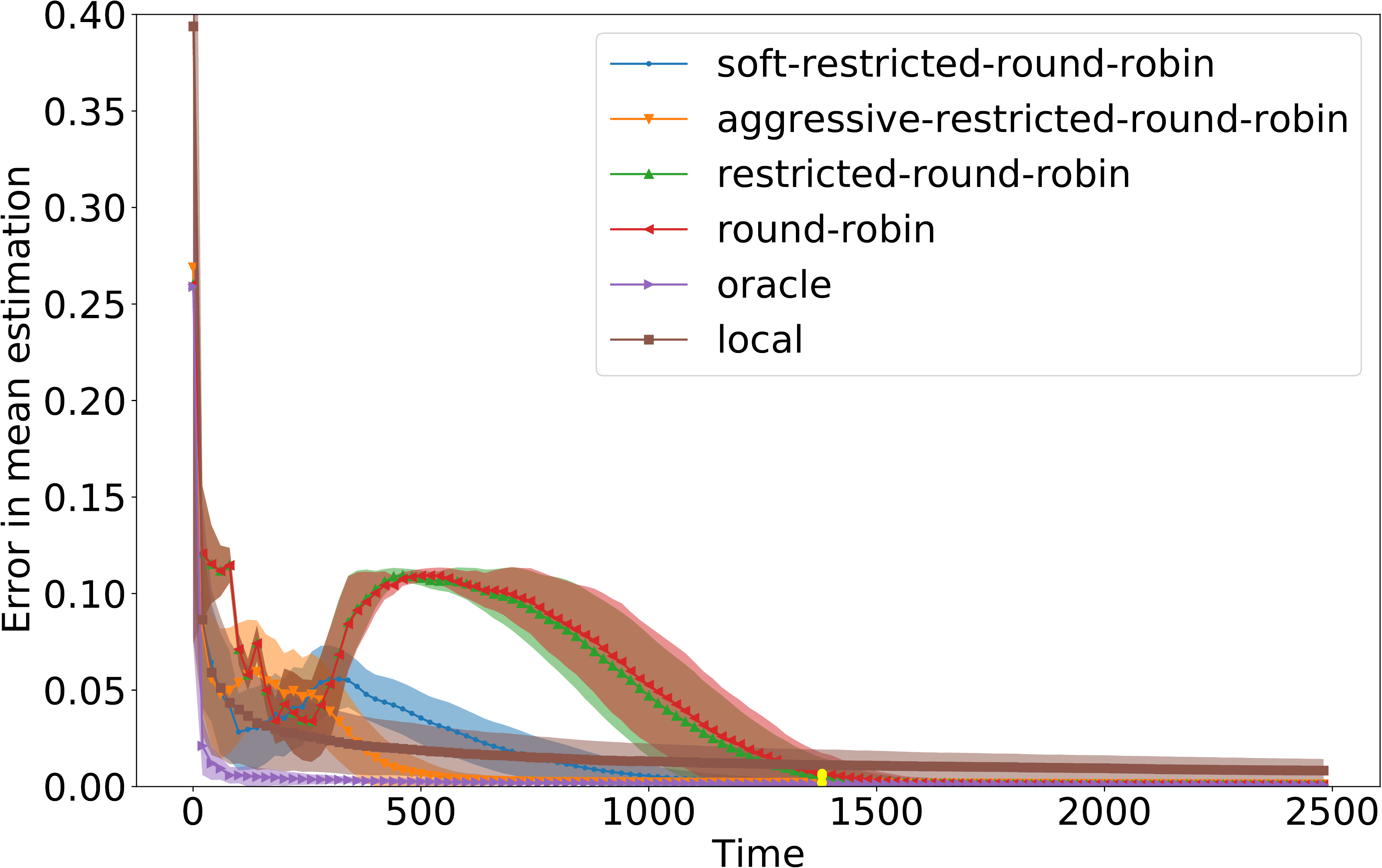}
        \label{fig:estimation_gaussian3_1}}
    \subfigure[Third class (mean $0.8$)]{\includegraphics[scale=0.20]
    {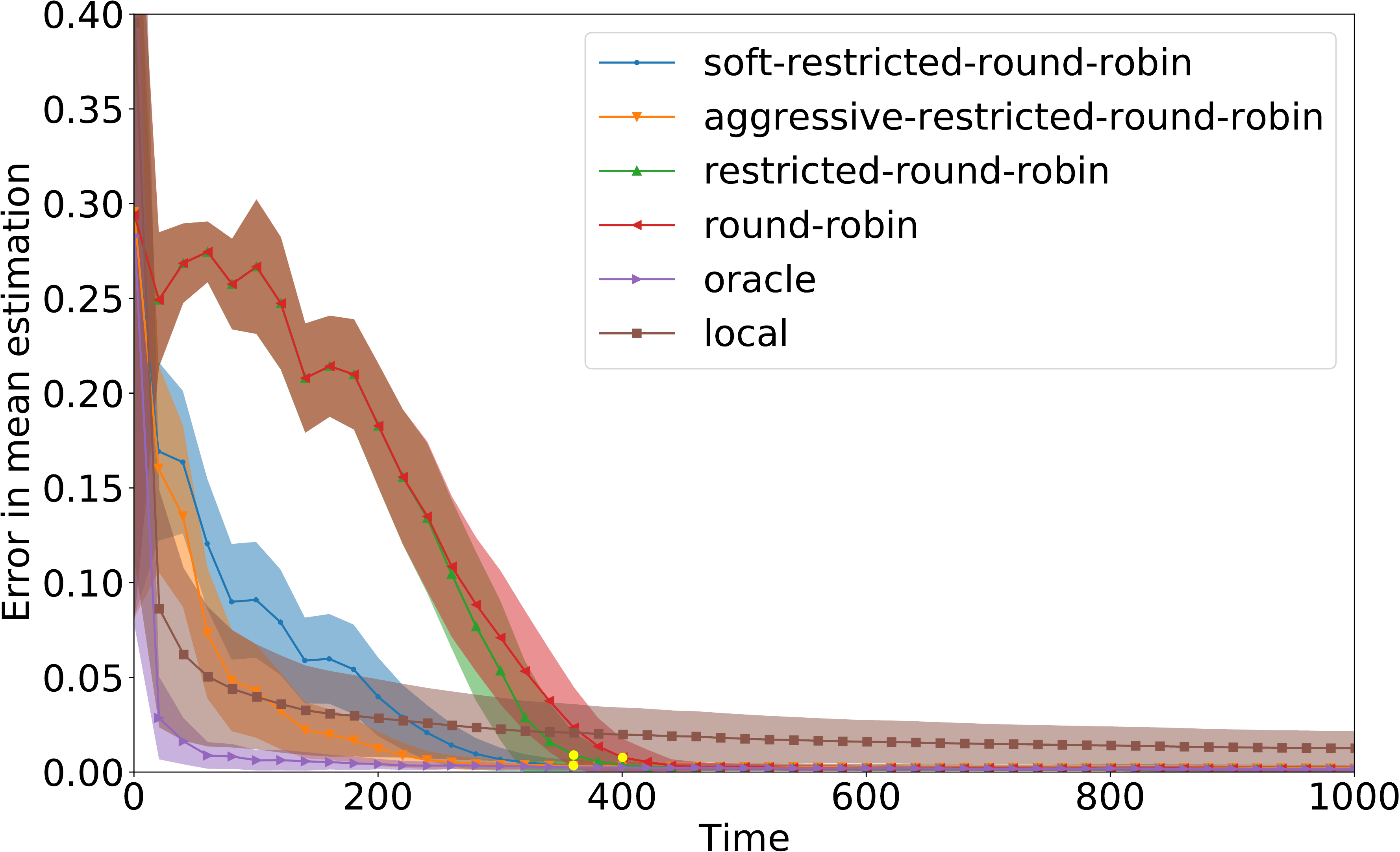}
        \label{fig:estimation_gaussian3_2}}
    \caption{Error in mean estimation for the 3-class problem (Gaussian
    distributions with true means 0.2, 0.4, 0.8).
    Note that the time scale is different for the third class to show the
    relevant details more clearly.
  \label{fig:estimation_gaussian3}}
\end{figure}

\begin{figure}[t]
	\centering
	\subfigure[\rev{Class estimation precision for $0.2$ vs $0.4$}]{\includegraphics
	[scale=0.18]
		{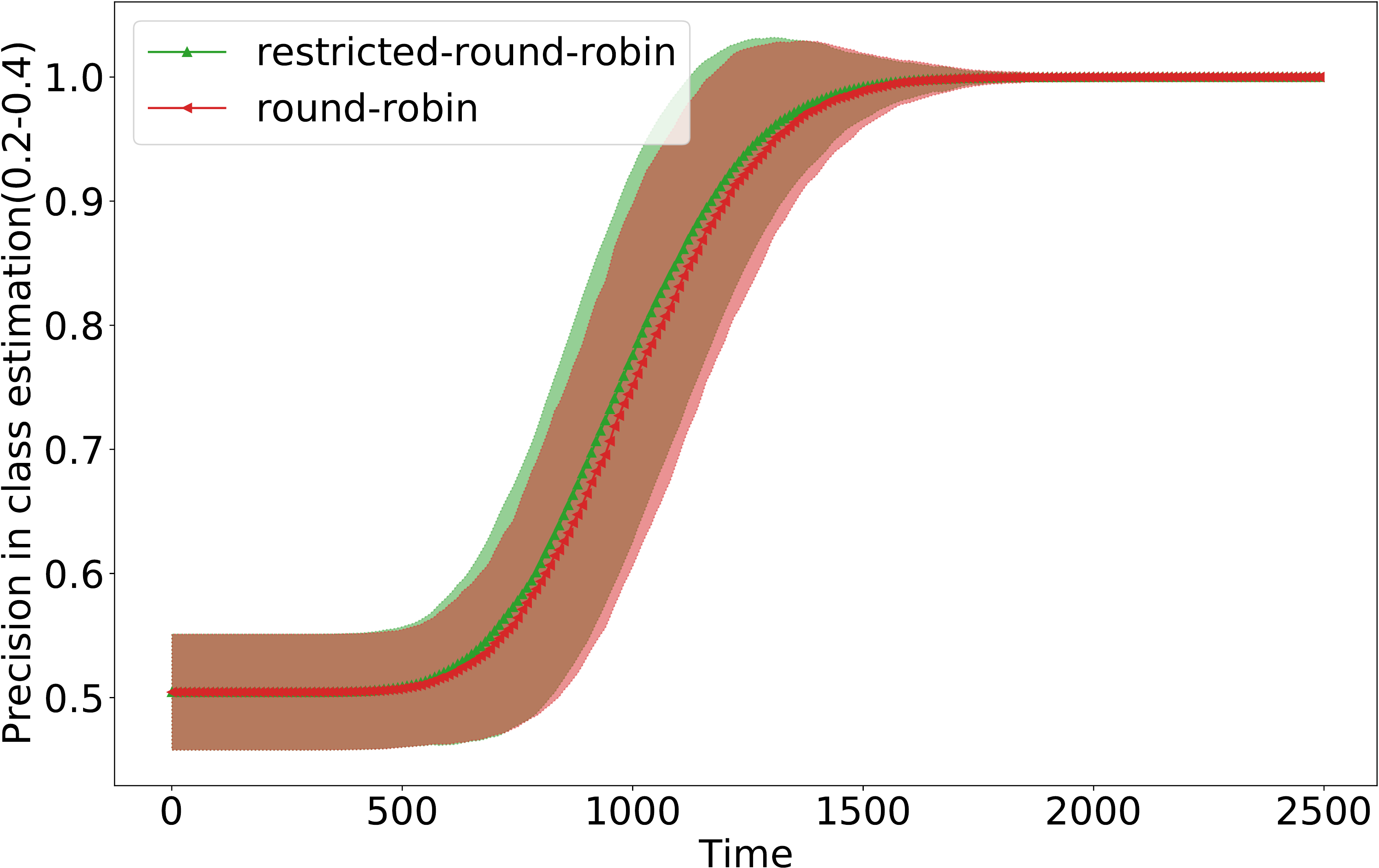}
		\label{fig:identification_0}}
		\hspace*{.5cm}
	\subfigure[\rev{Class estimation precision for $0.2$ vs $0.8$}]
	{\includegraphics
	[scale=0.18]
		{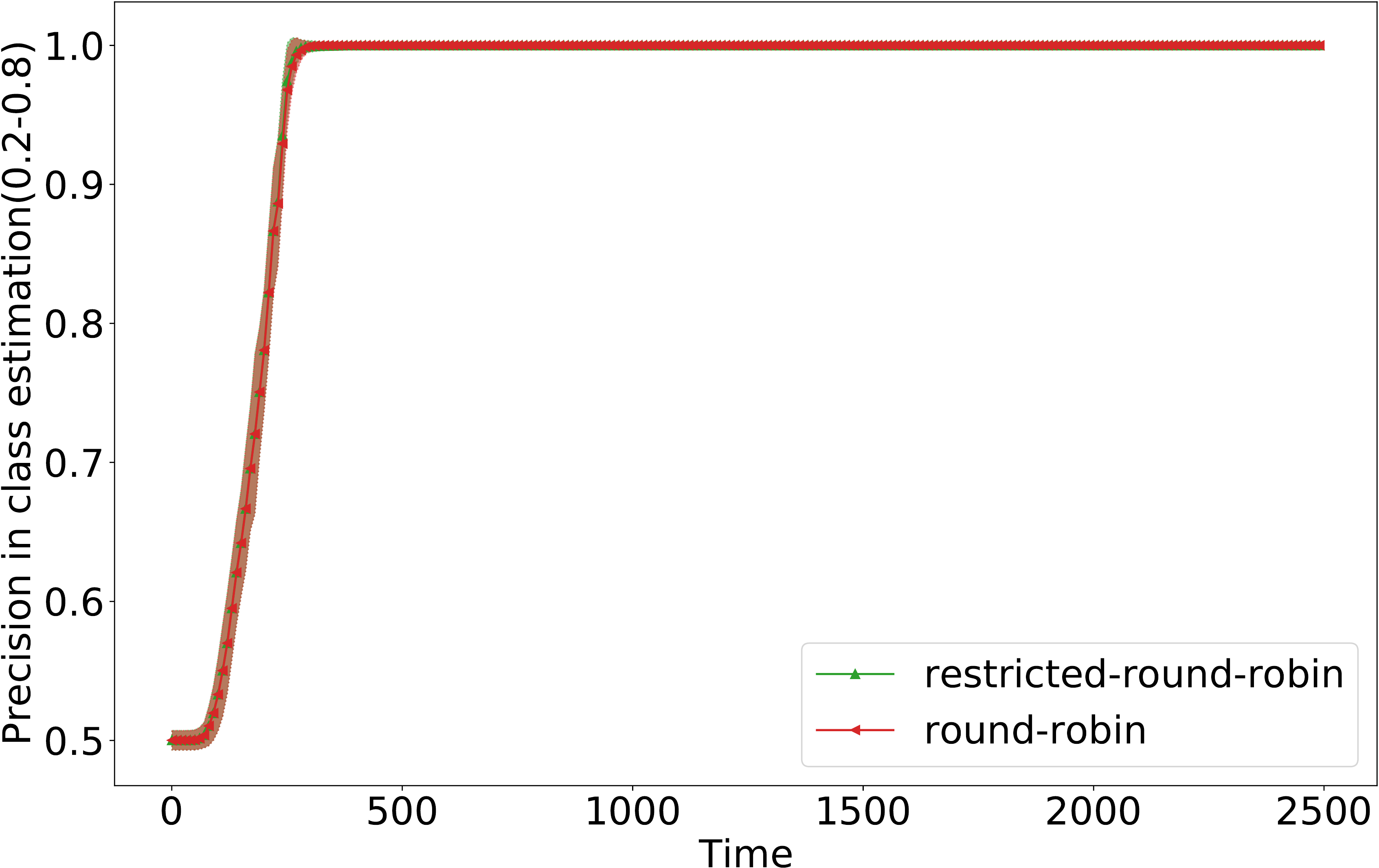}
		\label{fig:identification_1}}
	\subfigure[\rev{Class estimation precision for $0.4$ vs $0.8$}]{\includegraphics
	[scale=0.18]
		{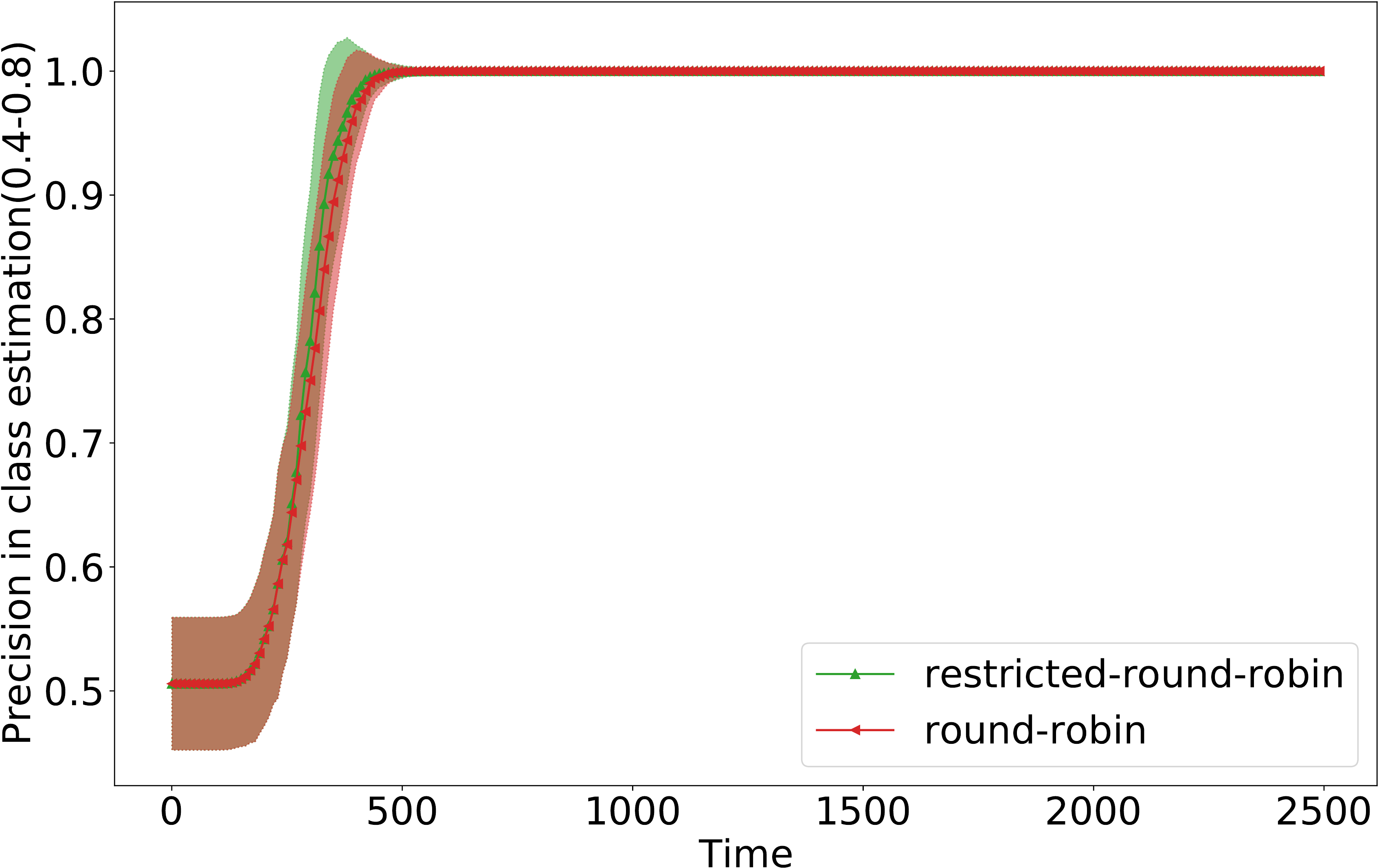}
		\label{fig:identification_2}}
	\caption{\rev{Class estimation precision across time for each pair of
	classes on the 3-class problem.}
		\label{fig:identification}}
\end{figure}

\section{Additional Experimental Results}

In this section we provide additional illustrative results to better
understand different aspects of our \algo algorithm.

\subsection{Per-Class behavior on the 3-class problem}
\label{sec:per-class}

For the 3-class problem described in the main text, we provide complementary
figures to show the error in mean estimation
\textit{in each
class separately}. These plots are shown in
Figure~\ref{fig:estimation_gaussian3}.
We can see that the average class identification time (represented by yellow
dots) is
different for different classes. For instance, as the gap between the
class with mean 0.8 and the other two is larger, this class requires less
samples to be identified. Indeed, the average class identification time is
less than $t=400$ for that class
(Figure~\ref{fig:estimation_gaussian3_2}), while it is about $t=1400$ for
the other two (Figures~\ref{fig:estimation_gaussian3_0}-\ref{fig:estimation_gaussian3_1}). Therefore, agents from the class with mean
$0.8$ reach a highly accurate estimate much faster than agents from other
classes.

\rev{We also show the class estimation precision across each pair of classes
in Figure~\ref{fig:identification}. We see that classes $0.2$
and $0.8$, who have the biggest gap, are separated first
(Figure~\ref{fig:identification_1}), then $0.4$ and $0.8$ (Figure~
\ref{fig:identification_2}) and finally $0.2$ and $0.4$ (Figure~
\ref{fig:identification_0}).}

\rev{Finally, for mean estimation, we provide the per-class counterparts of
Table~\ref{table:epsilonthreshold} in
Tables~\ref{table:epsilonthreshold-class0}-\ref{table:epsilonthreshold-class2}. In
line with previous results on class estimation, agents of class 0.8 are the
ones that converge faster to the desired mean estimation accuracy.
Interestingly, observe that for $\epsilon=0.1$, agents of class 0.2 converge
almost as fast as those from class 0.8: this is because they do
not need to eliminate all agents from class 0.4 to reach
this accuracy. This illustrates how our approach can naturally adapt to the
gaps and desired estimation accuracy.}

 \begin{table*}[t]
    \centering
    \small
    \rev{
	\begin{tabular}{lllll}
		\toprule
		Algorithm & \multicolumn{2}{c}{$\mathrm{conv}(0.1)$} &
		\multicolumn{2}{c}{$\mathrm{conv}(0.01)$}\\
		& avg/std & max & avg/std & max\\
		\midrule
		Round-Robin (RR) & $289 \pm 92$& $594$&$1207 \pm 170$ &$1782$\\
		Restricted-RR & $271 \pm 77$& $568$ & $1194 \pm 177$&
		$1857$ \\
		Soft-Restricted-RR & $111 \pm 30$& $239$ &
		$761 \pm 173$& $1379$\\
		Aggressive-Restricted-RR& $74 \pm 40$& $\mathbf{215}$& $
		\mathbf{382 \pm 95}$& $\mathbf{830}$ \\
		\midrule
		Local & $\mathbf{41 \pm 40}$& $288$& $4548 \pm 4075$& $28616$
		\\
		\midrule
		Oracle & $\mathit{6 \pm 4}$& $\mathit{15}$& $\mathit{100 \pm
		63}$& $\mathit{246}$ \\
		\midrule
		\midrule
		Theoretical Restricted-RR ($\tau_a$) & \multicolumn{2}{c}{$3878$}‌& 
		\multicolumn{2}{c}{$33406$} \\
		Theoretical Local ($\tau_a$)& \multicolumn{2}{c}{$885$}& \multicolumn{2}{c}
		{$100216$}\\
		\bottomrule
	\end{tabular}
	}
    \caption{\rev{Empirical convergence times (see
    Eq.~\ref{eq:convergencetime3})
    for class 0.2 of
    different algorithms on the 3-class problem (Gaussian
   distributions with true means $0.2$, $0.4$, $0.8$) for a target estimation error of $\epsilon=0.1$ 
    (unfavorable regime) and
    $\epsilon=0.01$ (favorable regime). We report
   the average, standard deviation and maximum across agents and runs.
   We also
   report the high-probability mean
   estimation times $\tau_a$ given by our theory for \RRR and \Local.}}
    \label{table:epsilonthreshold-class0}
\end{table*}

 \begin{table*}[t]
    \centering
    \small
    \rev{
	\begin{tabular}{lllll}
		\toprule
		Algorithm & \multicolumn{2}{c}{$\mathrm{conv}(0.1)$} &
		\multicolumn{2}{c}{$\mathrm{conv}(0.01)$}\\
		& avg/std & max & avg/std & max\\
		\midrule
		Round-Robin (RR) & $721 \pm 160$& $1239$&$1232 \pm 181$&$2088$
		\\
		Restricted-RR & $708 \pm 158$& $1175$& $1216 \pm
		183$& $1925$ \\
		Soft-Restricted-RR & $\mathbf{31 \pm 33}$& $340$ & $829
		\pm 188$& $1432$\\
		Aggressive-Restricted-RR& $35 \pm 42$& $340$& $
		\mathbf{420 \pm 92}$& $\mathbf{832}$ \\
		\midrule
		Local & $41 \pm 39$& $\mathbf{277}$& $4482 \pm 3846$&
		$26488$ \\
		\midrule
		Oracle & $\mathit{4 \pm 3}$& $\mathit{19}$& $\mathit{99 \pm
		70}$ & $\mathit{303}$\\
		\midrule
		Theoretical Restricted-RR ($\tau_a$)& \multicolumn{2}{c}{$3878$}‌& \multicolumn{2}{c}{$33406$} \\
		Theoretical Local ($\tau_a$)& \multicolumn{2}{c}{$885$}& \multicolumn{2}{c}
		{$100216$}\\
		\bottomrule
	\end{tabular}
	}
    \caption{\rev{Empirical convergence times (see
    Eq.~\ref{eq:convergencetime3})
    for class 0.4 of
    different algorithms on the 3-class problem (Gaussian
   distributions with true means $0.2$, $0.4$, $0.8$) for a target estimation error of $\epsilon=0.1$ 
    (unfavorable regime) and
    $\epsilon=0.01$ (favorable regime). We report
   the average, standard deviation and maximum across agents and runs.
   We also
   report the high-probability mean
   estimation times $\tau_a$ given by our theory for \RRR and \Local.}}
    \label{table:epsilonthreshold-class1}
\end{table*}

 \begin{table*}[t]
    \centering
    \small
    \rev{
	\begin{tabular}{lllll}
		\toprule
		Algorithm & \multicolumn{2}{c}{$\mathrm{conv}(0.1)$} &
		\multicolumn{2}{c}{$\mathrm{conv}(0.01)$}\\
		& avg/std & max & avg/std & max\\
		\midrule
		Round-Robin (RR) & $271 \pm 35$& $401$&$379 \pm
		38$&$538$ \\
		Restricted-RR & $266 \pm 31.25$& $384$ & $342 \pm 38$ & $
		\mathbf{526}$\\
		Soft-Restricted-RR & $100 \pm 36$& $232$ & $280 \pm 78$&
		$958$\\
		Aggressive-Restricted-RR& $57 \pm 24$& $\mathbf{172}$& $
		\mathbf{222 \pm 90}$& $958$ \\
		\midrule
		Local & $\mathbf{41 \pm 39}$& $289$& $4434 \pm 3883$ & $27470$\\
		\midrule
		Oracle & $\mathit{6 \pm 4}$& $\mathit{17}$& $\mathit{95 \pm
		56}$& $\mathit{208}$ \\
		\midrule
		\midrule
		Theoretical Restricted-RR ($\tau_a$)& \multicolumn{2}{c}{$1085$‌}& \multicolumn{2}{c}{$33406$} \\
		Theoretical Local ($\tau_a$)& \multicolumn{2}{c}{$885$}& \multicolumn{2}{c}
		{$100216$}\\
		\bottomrule
	\end{tabular}
	}
    \caption{\rev{Empirical convergence times (see
    Eq.~\ref{eq:convergencetime3})
    for class 0.8 of
    different algorithms on the 3-class problem (Gaussian
   distributions with true means $0.2$, $0.4$, $0.8$) for a target estimation error of $\epsilon=0.1$ 
    (unfavorable regime) and
    $\epsilon=0.01$ (favorable regime). We report
   the average, standard deviation and maximum across agents and runs.
   We also
   report the high-probability mean
   estimation times $\tau_a$ given by our theory for \RRR and \Local.}}
    \label{table:epsilonthreshold-class2}
\end{table*}

\subsection{Results on a 2-class problem}

\begin{figure}[t]
	\centering
	\subfigure[Class estimation precision over all agents]{\includegraphics[scale=0.2]
		{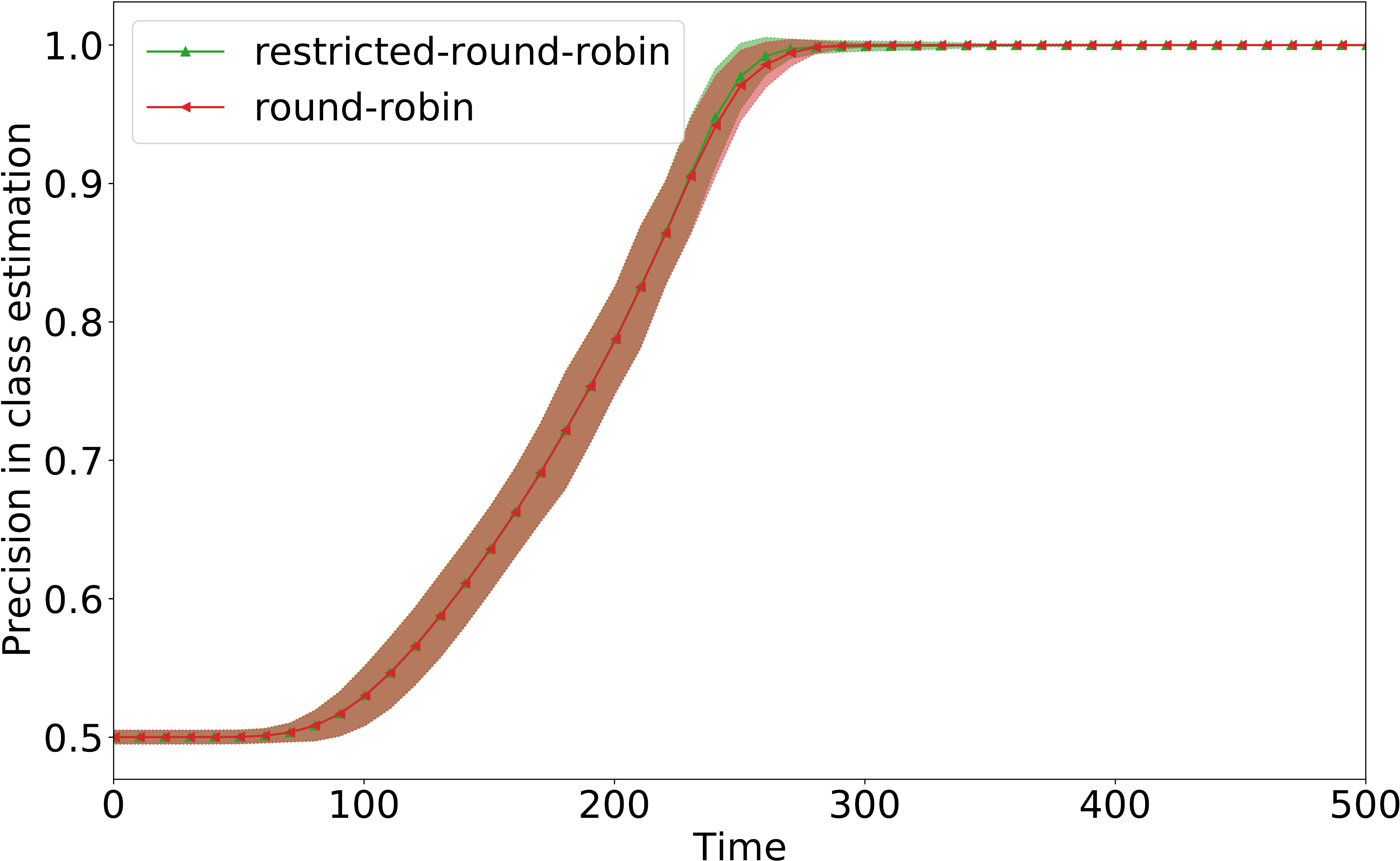}
		\label{fig:precision2}}
	\subfigure[Error in mean estimation over all agents]{\includegraphics
    [scale=0.2]
		{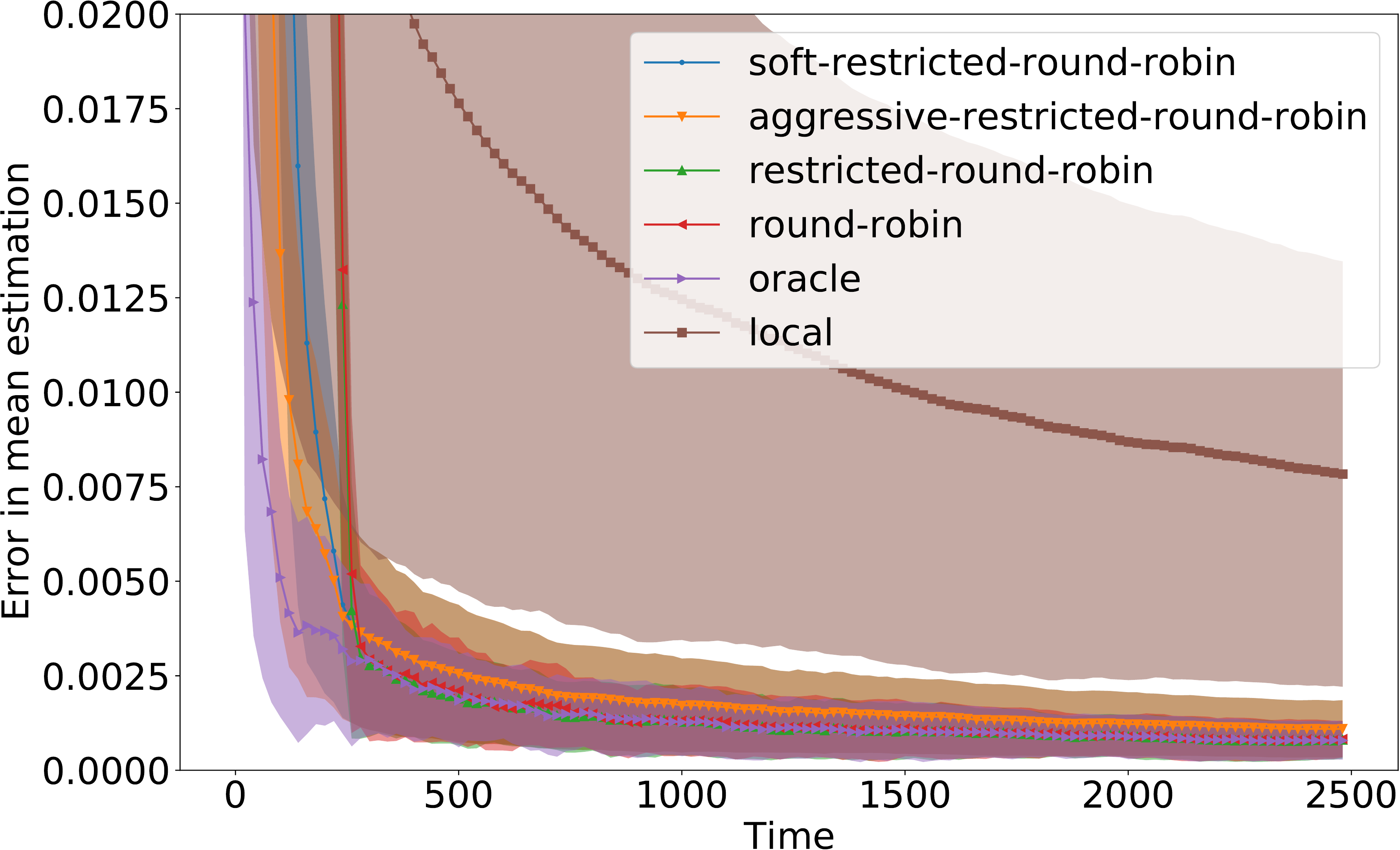}
		\label{fig:network_mean_estimation}}
	\caption{Results for the 2-class problem (Gaussian
    distributions with true means 0.2 and 0.8).
    \label{fig:2-class}}
\end{figure}

We experiment with a 2-class problem generated in the same way as the 3-class
 problem considered in the main text, except that the means are chosen among $
\{0.2, 0.8\}$. This makes the problem easier since the gap between the two
 classes corresponds to the largest gap in the 3-class problem. The
 results shown in Figure~\ref{fig:2-class} reflect this: agents correctly
 identify their class and reach highly accurate mean estimates much faster
 than in the 3-class problem. Consequently, the improvement compared to \Local
 is even more significant and our approach almost matches the performance of
 \Oracle. We omit the per-class figures as they are
 essentially the same as Figure~\ref{fig:network_mean_estimation}.

%%% Local Variables:
%%% mode: latex
%%% TeX-master: "supp"
%%% ispell-local-dictionary: "english"
%%% End:

\end{document}